\pdfoutput=1
\documentclass{article}

\PassOptionsToPackage{numbers, sort&compress}{natbib}

    \usepackage[final]{neurips_data_2021}

\setcitestyle{open={(},close={)}}
\usepackage[utf8]{inputenc} 
\usepackage[T1]{fontenc}    
\usepackage{hyperref}       
\usepackage{url}            
\usepackage{booktabs}       
\usepackage{amsfonts}       
\usepackage{amsthm}
\usepackage{nicefrac}       
\usepackage{microtype}      
\usepackage{listings}
\usepackage{appendix}
\usepackage{natbib}
\setcitestyle{square}
\usepackage{xspace}
\usepackage{graphicx}
\usepackage{bmpsize}
\usepackage{float}
\usepackage{subcaption}
\usepackage{tabularx} 
\usepackage{enumitem} 
\usepackage{dcolumn}
\newcolumntype{d}{D{.}{.}{-1}}
\newcolumntype{z}{b{2mm}D{.}{.}{-1}}
\newcolumntype{s}{D{/}{/}{-1}}
\usepackage{wrapfig}
\usepackage{colortbl}
\newtheorem{observation}{Observation}
\usepackage{mdframed}

\usepackage[usenames, dvipsnames, svgnames]{xcolor}
\definecolor{darkblue}{rgb}{0, 0, 0.5}
\hypersetup{colorlinks=true, citecolor=darkblue, linkcolor=darkblue, urlcolor=darkblue}

\definecolor{Gray}{gray}{0.95}

\newcommand\pyline[1]{\lstinline[language=Python,basicstyle=\ttfamily]{#1}}
\lstnewenvironment{pydisplay}{\lstset{language=Python,basicstyle=\ttfamily}}{}
\lstnewenvironment{pylittledisplay}{\lstset{language=Python}}{}
\lstnewenvironment{pyblock}{\lstset{language=Python,frame=single,basicstyle=\ttfamily,morekeywords={assert}}}{}
\lstnewenvironment{pysmall}{\lstset{language=Python,frame=single}}{}

\newcommand{\X}{\mathcal{X}}
\newcommand{\Y}{\mathcal{Y}}
\newcommand{\F}{\mathcal{F}}

\newcommand{\NP}{\mathrm{NP}}
\newcommand{\poly}{\mathrm{poly}}

\title{Programming Puzzles}

\author{Tal Schuster\\MIT
\And Ashwin Kalyan\\Allen Inst.~for AI
\And Oleksandr Polozov\\Microsoft Research \And Adam Tauman Kalai\\Microsoft Research}

\begin{document}

\maketitle
\begin{abstract}
We introduce a new type of programming challenge called programming \textit{puzzles}, as an objective and comprehensive evaluation of program synthesis, and release an open-source dataset of Python Programming Puzzles (P3).\footnote{\url{https://github.com/microsoft/PythonProgrammingPuzzles}} Each puzzle is defined by a short Python program $f$, and the goal is to find an input which makes $f$ return \pyline{True}. The puzzles are objective in that each one is specified entirely by the source code of its verifier $f$, so evaluating $f$ is all that is needed to test a candidate solution.  They do not require an answer key or input/output examples, nor do they depend on natural language understanding. The dataset is comprehensive in that it spans problems of a range of difficulties and domains, ranging from trivial string manipulation problems, to classic programming puzzles (e.g., Tower of Hanoi), to interview/competitive-programming problems (e.g., dynamic programming), to longstanding open problems in algorithms and mathematics (e.g., factoring). We develop baseline enumerative program synthesis, GPT-3 and Codex solvers that are capable of solving puzzles---even without access to any reference solutions---by learning from their own past solutions. Codex performs best, solving up to 18\% of 397 test problems with a single try and 80\% of the problems with 1,000 tries per problem. In a small user study, we find a positive correlation between puzzle-solving performance and coding experience, and between the puzzle difficulty for humans and AI solvers. Therefore, further improvements on P3 could have a significant impact on many program synthesis areas.


\end{abstract}

\section{Introduction}

Puzzles are often used to teach and evaluate human programmers. Classic puzzles such as the Tower of Hanoi teach fundamental concepts such as recursion. Programming competition problems, also referred to as puzzles \cite{laaksonen2020guide}, evaluate a participant's ability to apply these concepts. Puzzles are also used to evaluate programmers in job interviews, and puzzles such as the RSA-factoring challenge test the limits of state-of-the-art algorithms. Each of these types of puzzles is described in its own format, often in a natural language such as English. Evaluations often include a hidden test set. 

We introduce a novel puzzle representation called a programming puzzle or simply a \textit{puzzle}, which captures the essence of these challenges in a form convenient for machines and programmers. At a minimum, a puzzle is specified by a function $f(y)$, and the goal is to find $y$ such that $f(y)=\mathtt{True}$. 
More generally, a puzzle can include input variables $x$. Then, the puzzle can be seen as an output \emph{verifier} $f(y, x)$ that validates $y$. The answer $y$ is typically the output of a synthesized program $g$. 
In order to find  $g(x) \rightarrow y$, a \textit{synthesizer} is given the source code of $f$ (and possibly $x$), with the goal of generating a program $g$ such that $f(g(x), x)=\mathtt{True}$. 
Importantly, puzzles make for an objective and explicit programming evaluation based solely on code with no formal requirement for input/output examples, natural language descriptions, or reference solutions.  


\begin{figure}[t]
 \vspace{-1.3\baselineskip}
\small
\begin{pyblock}
# Find a string that when reversed and concatenated with "world" gives "Hello world"
def f1(y: str):  
    return y[::-1] + "world" == "Hello world"

# Tower of Hanoi, often teaches recursion. Move [i, j] means move top disk on tower i to j, with 1 $\le$ i,j $\le$ 3
def f2(moves: List[List[int]], num_disks=8):
    state = [1] * num_disks  # All disks start at tower 1.               
    for [i, j] in moves:  
        assert state.index(i) <= (state + [1, 2, 3]).index(j), "bigger disk on top"
        state[state.index(i)] = j  # Move smallest disk from tower i to tower j.
    return state == [3] * num_disks  # All disks must end on tower 3.

# Find a non-trivial integer factor d of a large number n
def f3(d: int, n=100433627766186892221372630609062766858404681029709092356097): 
    return 1 < d < n and n % d == 0


\end{pyblock}
    \vspace{-1\baselineskip}
    \caption{Programming puzzles ranging from trivial to longstanding open algorithmic challenges in multiple domains. \pyline{f1} is solved by \pyline{y="Hello "[::-1]}, a recursive program (see Figure \ref{fig:hanoi} on page \pageref{fig:hanoi}) outputting 255 moves solves \pyline{f2}, and \pyline{f3} requires computational number theory algorithms.}
    \label{fig:examples}
    \vspace{-1\baselineskip}
\end{figure}

Puzzles may have multiple valid outputs $y$ and some puzzles, even if very short, are extremely challenging. Figure \ref{fig:examples} illustrates three puzzles that are diverse in domain, difficulty, and algorithmic tools. The first puzzle is an easy (for humans)
puzzle that tests one’s understanding of basic syntax and properties of strings. The second is  the quintessential example of recursion, and the last is a hard problem
requiring advanced algorithms such as the quadratic sieve.


%
%


We also release a growing open-source Python Programming Puzzles dataset, called P3, which is already comprehensive in terms of difficulty, domain, and algorithmic tools. This dataset unifies many of the types of puzzles mentioned above. While P3's puzzles are all specified in Python, solution programs $g$ can be written in any language, or simulated by neural networks.

As describe in \S\ref{sec:the_dataset}, P3 also contains numerous classic puzzles; optimization puzzles such as solving a linear programming; graph puzzles such as shortest path; and competitive-programming problems. The most difficult puzzles involve longstanding open problems such as learning parity with noise~\citep{blum2003}; factoring~\citep{theory_of_numbers_book,Kaliski2005}; or finding a cycle in the $3n+1$ process which would disprove the Collatz conjecture~\citep{collatzConjecture}. Thus, if AI were to surpass human-level performance on this dataset, it would lead to breakthroughs on major open problems. The P3 puzzles were inspired by sources such as Wikipedia, algorithms books, programming competitions, and other program synthesis datasets. This dataset is growing as an open-source project, and anyone can add a puzzle by simply writing a function \pyline{f}.

One motivation for programming puzzles is that improvements in solving puzzles may lead to performance gains at other tasks. 
Recently, neural Language Models (LMs) have advanced the state-of-the-art of AI systems performance in offering code completions and synthesizing source code in general-purpose programming languages, such as Python, based on English descriptions~\citep{chen2021evaluating, austin2021program}. While such systems represent a major advance over prior approaches,~\citet{chen2021evaluating} point out that they also reproduce elementary programming mistakes. Figure~\ref{fig:autopilot} illustrates how state-of-the-art GitHub Copilot~\citep{copilot} solves a complex problem, handles the ambiguity of English, and yet makes elementary errors. Performance gains in solving programming puzzles may result in fewer errors or solving more sophisticated algorithms problems in downstream tasks such as code completion.

Puzzles are \textit{objective}, meaning that it is easy to unambiguously evaluate whether one's own answer is valid without consulting an answer key.  This evaluation also allows bootstrapping, even on test puzzles without gold solutions. Given a set of puzzles $(f_i, x_i)$, one can attempt to solve them with solutions $g_i$, determine with certainty which solutions are correct, and use those to improve one's ability to solve the remaining puzzles~\citep{dechter2013bootstrap}. 
Inspired by success in playing games~\citep{tesauro1995temporal,alphaZero}, self-training has also proven useful in program synthesis~\citep[see, e.g.,][]{balog2016deepcoder, christakopoulou2018glass}. Other commonly used representations, including natural language descriptions or Programming by Example (PbE), have inherent ambiguity. See Appendix~\ref{ap:compare} for a comparison of a competition problem represented in English and as a puzzle. 

From a theoretical point of view, as we shall discuss, objectivity can be formalized as the complexity class $\NP$ of non-deterministic polynomial-time decision problems. Moreover, the puzzle decision problem is $\NP$-complete, meaning  puzzles can readily express any $\NP$ problem, including polynomial-time problems and other $\NP$-complete problems such as Boolean satisfiability.



We compare several enumerative random forest and Transformers-based top-down solvers, as well as GPT-3 and Codex LM solvers with different prompt types (e.g., zero/few-shot and with/without English descriptions). 
In our experiments, without access to any reference solutions, only utilizing self-training bootstrapping, our enumerative models solved up to 43\% more P3 problems than a naive brute force baseline. Our LM solvers were able to solve many of the puzzles, given enough tries. 


\begin{figure}
\centering
\includegraphics[width=\textwidth]{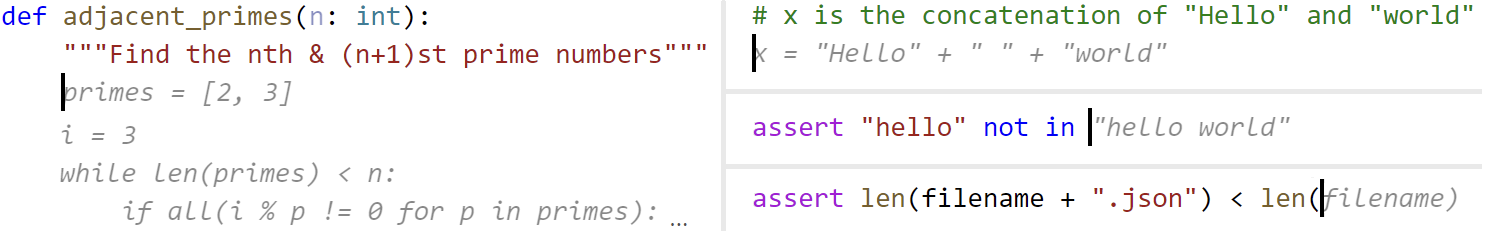}
\caption{GitHub Copilot code completion examples (in gray). Left: Copilot
correctly implements a seven-line function. Top right: 
the completion adds a space character that may or may not have been intended by the user. Middle and bottom right: errors indicating a lack of basic understanding.}
\label{fig:autopilot}
\end{figure}

To address the questions of whether puzzles measure programming proficiency and how puzzle difficulty compares between humans and computers, we performed a small user study. Puzzles were accessible and enjoyable for programmers with varying levels of experience. While both GPT-3 and enumerative techniques can solve a fraction of the puzzles, human programmers outperform them. 
For example, bootstrapping GPT-3 with up to 10K tries solved 60\% of the puzzles, lower than both beginner and experienced participants that solved 76\% and 87\% puzzles on average, respectively.
Overall, we find perceived puzzle difficulty to scale similarly for both humans and AI.
The main contributions of this paper are introducing:\vspace{-7pt}
\begin{enumerate}[noitemsep]
    \item programming puzzles: a new type of problem suitable for algorithmic problem-solving (for both machines and humans);\vspace{2pt}
    \item P3, an open-source dataset of puzzles covering diverse domains and difficulties; and\vspace{2pt}
    \item an evaluation of humans and baselines demonstrating that puzzles can be used to measure algorithmic problem-solving progress. 
\end{enumerate}\vspace{-4pt}

Progress in code completion is rapid---even between the time of submission of this paper and its publication, an API to Codex (a GPT-3 model fine-tuned for code completion) was released \citep{chen2021evaluating}. Our evaluation does in fact show significant improvements of Codex over other baselines.
\section{Problem formulation}

Programs, inputs and outputs can all be formally represented as strings, where $\Sigma^*$ is the set of finite strings over alphabet $\Sigma$. The set of verification functions is denoted by $\F \subseteq \Sigma^*$, with inputs and outputs $\X, \Y \subseteq \Sigma^*$, respectively. A puzzle is defined by pair  $(f, x) \in \F \times \X$ and the result of running verifier $f$ on output $y \in \Y$ is denoted $f(y, x) \in \{0,1\}$. Output $y \in \Y$ is \emph{valid} if it \emph{satisfies} $f(y, x)=1$, i.e., $f$ outputs $1$ when run on $(y, x)$, within a specified amount of time. To ensure that puzzles can be \textit{quickly} verified, it is necessary to upper-bound the time required for puzzle verification. This ensures that the puzzle decision problem, namely the problem of determining whether, given $f,x$, there is $y$ such that $f(y,x)=1$, is in the complexity class $\NP$. Formally, the puzzle decision problem is, given strings $f$ and $x$ denoting the puzzle (represented as, say, a Turing machine) and input, and a timeout $t$, does the puzzle output 1 in time $\leq t$. See Appendix \ref{sec:NP} for further details.

A \textbf{solver} takes $n$ puzzles and timeouts $(f_1, x_1, t_1), \ldots, (f_n, x_n, t_n)$, and produces outputs $y_i$ to as many puzzles as it can within a time bound $T$. Of course $T \gg \sum t_i$ is significantly larger than the verification timeouts. Formally, the  \textit{score}  of solver $S: \F^n \rightarrow \X^n$ is the number of puzzles $f_i$ for which $f_i(y_i, x_i)$ outputs 1 in time $\leq t_i$. Although we do not study it in this paper, it would be natural to assign different \textit{values} to different puzzles. For example, solving open problems such as finding a Collatz cycle or factoring the largest RSA challenge integer (currently unsolved, with a \$200,000 prize offered), should be of greater value than solving a simple hello-world puzzle.

It is convenient, though not required, to solve puzzles by outputting a program $g$ which, when run, computes output $y=g(x)$. Such a program is called a \textbf{solution} $g$. Short solutions may have long outputs, e.g., the puzzle \pyline{(f=lambda y: len(y) == x, x=1000000)} requires a string of length one million as solution \pyline{g=lambda x: 'a' * x}. In this example, $y=g(1000000)$ is a valid output of length one million. Of course, another solution would be to explicitly write a string of length one million in the code, though this implementation may not pass a human code review. In the dataset and this paper, we provide solutions since they may be significantly shorter. Many puzzles fit a single \textbf{problem} template, meaning they share the same verifier $f$ but have different inputs $x$. Thus a dataset may have many more puzzles than problems.
\section{The P3 dataset}\label{sec:the_dataset}
P3 uses Python, the de facto language of ML research, as the programming language for specifying puzzles. At the time of publication, P3 currently has 397 problems, summarized in Table \ref{tab:data_stats}.  The latest dataset can be generated by simply running \pyline{make_dataset.py} in the repository.
More puzzles may be created by increasing the number of puzzles per problem argument, though most experiments in this paper use only one puzzle per problem.
Every puzzle is described by a function with a required typed argument (i.e., the candidate output) that returns \pyline{True} upon success. Since Python is not type-safe, we add type-checking to ensure that outputs match the declared type. Figure \ref{fig:typecheck} on page \pageref{fig:typecheck} illustrates a puzzle where type checking is important.

We also provide code for serializing Python objects to and from strings in a json-like format, so that programs implemented in any language can produce outputs. Moreover, strings are universal in that they can encode arbitrary Python objects including functions, as in the \href{https://github.com/microsoft/PythonProgrammingPuzzles/blob/main/puzzles/README.md#quine}{Quine}
puzzle \pyline{(lambda quine: eval(quine) == quine)}\footnote{GPT-3 generated a 5-character solution to the quine puzzle while the authors' solution was 88 characters.} motivated by the classic challenge of writing a program that outputs its own source code. As evaluation of the string \pyline{quine} can lead to an infinite loop, this puzzle illustrates the necessity of the evaluation timeout $t$ for attempted solutions.

While not necessary for evaluation (since puzzles are self-contained)
we follow the common practice of programming competitions and
provide a reference solution to most (over 90\%) of the puzzles. Some puzzles have more than one solution. A handful of puzzles represent major open problems in computer science and mathematics including \href{https://github.com/microsoft/PythonProgrammingPuzzles/blob/main/puzzles/README.md#factoring}{Factoring} (and \href{https://github.com/microsoft/PythonProgrammingPuzzles/blob/main/puzzles/README.md#discretelog}{Discrete Log}), \href{https://github.com/microsoft/PythonProgrammingPuzzles/blob/main/puzzles/README.md#plantedclique}{Planted Clique}, \href{https://github.com/microsoft/PythonProgrammingPuzzles/blob/main/puzzles/README.md#learnparitywithnoise}{Learning Parity with Noise}, \href{https://github.com/microsoft/PythonProgrammingPuzzles/blob/main/puzzles/README.md#GraphIsomorphism}{Graph Isomorphism}, and finding a \href{https://github.com/microsoft/PythonProgrammingPuzzles/blob/main/puzzles/README.md#collatzcycleunsolved}{Collatz cycle},\footnote{The solution to this problem would disprove the Collatz conjecture that is believed to be true, but no proof has been found yet. Therefore, if the conjecture is true, the maximum attainable score in P3 is $<100\%$.} as described in Appendix \ref{ap:open}. We also provide English descriptions for each puzzle in the dataset to support research involving natural language.
Appendix \ref{ap:compare} compares programming competition problems to puzzles.

\paragraph{Creation process.}
The following sources were used for identifying possible puzzles:\vspace{-7pt}
\begin{itemize}[leftmargin=*, noitemsep]
    \item Wikipedia, specifically the  
    \href{https://en.wikipedia.org/wiki/Category:Logic_puzzles}{Logic puzzles} category, the 
    \href{https://en.wikipedia.org/wiki/List_of_unsolved_problems_in_mathematics}{List of unsolved problems in mathematics}, and the
    \href{https://en.wikipedia.org/wiki/List_of_algorithms}{List of algorithms}.\vspace{2pt}
    \item Competitions, primarily the competitive programming website \href{https://codeforces.com}{codeforces.com} but also a handful of problems from the \href{https://icpc.global}{International Collegiate Programming Contest} and the \href{https://en.wikipedia.org/wiki/International_Mathematical_Olympiad}{International Mathematical Olympiad} (IMO)--a high school mathematics competition.\vspace{2pt}
    \item Puzzles inspired by the HumanEval dataset used for evaluating Codex~\cite{chen2021evaluating}, added in v0.2.\vspace{2pt}
    \item The Python programming language itself, with trivial puzzles created to test understanding of basic functions, such as the the hello-world puzzle which tests string concatenation.
\end{itemize}

P3 is organized topically into modules listed in Table \ref{tab:data_stats}. These topics include domains such as  number theory, graph theory, chess puzzles, game theory, etc., as well as puzzles inspired by a specific source such as a specific programming competition. One finding in this paper is that many types of puzzles can be captured in spirit, if not exactly, as succinct puzzles. Common patterns include:\vspace{-7pt}
\begin{itemize}[leftmargin=*, noitemsep]
\item Problems that are \textit{naturally} puzzles. For instance, search problems such as the \href{https://github.com/microsoft/PythonProgrammingPuzzles/blob/main/puzzles/README.md#towersofhanoi}{TowerOfHanoi} (\pyline{f2}, Figure~\ref{fig:examples}) and \href{https://github.com/microsoft/PythonProgrammingPuzzles/blob/main/puzzles/README.md#slidingpuzzle}{SlidingPuzzle} simply test the sequence of moves to see if they lead to the goal state. \vspace{2pt}
\item Problems that have an equivalent natural puzzle. For instance, the standard definition of the factoring problem, namely factorizing an integer into its prime factors would require a puzzle that tests primality.  However the simpler problem of \href{https://github.com/microsoft/PythonProgrammingPuzzles/blob/main/puzzles/README.md#factoring}{finding any non-trivial integer factor}, \pyline{f3} in Figure \ref{fig:examples}, can be recursively called to solve the prime factorization problem. \vspace{2pt}
\item Optimization problems. Some such problems have equivalent natural puzzles, e.g., linear programming is well-known \cite{dantzig1951proof} to be equivalent to solving a zero-sum game which is the \href{https://github.com/microsoft/PythonProgrammingPuzzles/blob/main/puzzles/README.md#zerosum}{ZeroSum} puzzle. For others, such as \href{https://github.com/microsoft/PythonProgrammingPuzzles/blob/main/puzzles/README.md#longestmonotonicsubstring}{LongestMonotonicSubstring} or \href{https://github.com/microsoft/PythonProgrammingPuzzles/blob/main/puzzles/README.md#shortestpath}{ShortestPath}, we specify a bound $\theta$ on the objective, and the goal is to find a feasible $y$ with objective better than $\theta$. In order to generate $\theta$ (included in $x$), we first solve the optimization problem ourselves, but the puzzle generation code is not provided to the solvers.\vspace{2pt}
\item Problems that ask how many items in a certain set satisfy a given property, may be converted to problems that require an explicit enumeration of all such items. See for example the \href{https://github.com/microsoft/PythonProgrammingPuzzles/blob/main/puzzles/README.md#allpandigitalsquares}{AllPandigitalSquares} puzzle that requires all 174 pandigital perfect squares as input.\vspace{2pt}
\item Problems that involve game-playing can often be converted to puzzles. In chess, this includes the classic Eight Queens and Knights Tour search problems. A puzzles Mastermind involves exhibiting a winning strategy tree, and a nim puzzle involves beating a given computer opponent.\vspace{2pt}
\end{itemize}\vspace{-5pt}
In order to ensure that each puzzle is achieving its goals, the 
puzzle design process has a step which automatically tests for trivial solutions such as small integers or common strings. 

\paragraph{Exclusions.}
Many programming challenges do not make as good puzzles. First, simple \textit{translation} tasks, where goal is translate a sequence of steps described in natural language into a program, do not make good puzzles. Second, some challenges require problem-solving that is not easily expressed as a program. For example, computing the probability of rolling a total of 350 when rolling 100 dice relies on external knowledge about probability theory. Third, ``soft'' challenges involving natural language or images are not in NP and not easily verifiable. This includes challenges involving human commonsense or world knowledge about names, dates, or image classification. Finally, \textit{interactive} challenges do not make for good programming puzzles. Fortunately, several other benchmarks cover these latter two types of exclusions~\cite[see, e.g.,][]{andreas-etal-2020-task, talmor-etal-2019-commonsenseqa, kiela-etal-2021-dynabench, ILSVRC15, Lourie2021UNICORNOR, petroni-etal-2019-language, rajpurkar-etal-2016-squad, sap-etal-2019-social, schuster-etal-2021-get, wang-etal-2018-glue, zellers-etal-2018-swag, zellers2019vcr, alphaZero, alphago}.


\begin{table}[t]
    \centering
    \small
        \caption{Number of problems (and how many of them have at least one reference solution) per domain in P3 v0.2. The right two columns show the average size of puzzles and solutions, measured by the number of nodes in the Python AST.}
                \resizebox{0.65\columnwidth}{!}{%
    \begin{tabular}{l|dddd}
    \toprule
Domain & \multicolumn{1}{l}{Problems} & \multicolumn{1}{l}{Solutions} & \multicolumn{1}{l}{$|f|$} & \multicolumn{1}{l}{$|g|$}\\
\midrule
Algebra & 4 & 4  & 70& 172 \\
Basic & 23 & 23  & 54& 44 \\
Chess & 5 & 3  & 221& 186 \\
Classic puzzles & 23 & 23  & 101& 211 \\
Codeforces & 47 & 45  & 73& 70 \\
Compression & 2 & 2  & 126& 113 \\
Conways Game of Life & 3 & 2  & 189& 345 \\
Games & 7 & 7  & 225& 299 \\
Graphs & 12 & 11  & 105& 152 \\
HumanEval & 164 & 164  & 81& 62 \\
ICPC & 4 & 4  & 304& 569 \\
IMO & 6 & 6  & 173& 256 \\
Lattices & 2 & 2  & 70& 228 \\
Number Theory & 16 & 12  & 47& 68 \\
Probability & 5 & 5  & 85& 72 \\
Study & 30 & 30  & 40& 21 \\
Trivial inverse & 39 & 38  & 27& 30 \\
Tutorial & 5 & 5  & 27& 13 \\
\midrule
Total \# / Average size & 397 & 386  & 79 & 84 \\
\bottomrule
    \end{tabular}
    } %

    \label{tab:data_stats}
    \vspace{-1\baselineskip}
\end{table}

\paragraph{Growth process.}
The focus of this paper is in creating a framework with an initial dataset; and demonstrating its utility for developing and evaluating AI solvers.
As a GitHub repository, the dataset can grow over time in a standard manner with the ability to reference previous versions.
We plan to continue adding puzzles and hope that others will as well. Popular competitive-programming websites such as  \href{https://codeforces.com}{codeforces} may be a source of thousands of puzzles of varying difficulties. 

\section{Solvers}\label{sec:solvers}
In this section, we describe the models we develop as baselines for the dataset. 
We consider both solving problems independently and joint solvers that bootstrap from previously obtained solutions to find new ones.
We also consider both enumerative solvers that use standard techniques from program synthesis and LM solvers that use GPT-3 and Codex to solve puzzles. While a direct comparison between these two different approaches is difficult because they run on different hardware (the LMs call an API), we can still compare the relative difficulty with which they solve different puzzles, and also to human difficulty rankings among puzzles. 

\subsection{Enumerative solvers}
Following prior work~\citep{alon2020structural, menon2013machine,balog2016deepcoder, christakopoulou2018glass}, we develop models to guide the search for $g$ over the space of all possible functions. In particular, we implement a grammar that generates Abstract Syntax Trees (ASTs) for a large subset of Python. The grammar covers basic Python functionality and is described in Appendix \ref{sec:grammar_details}.
Specifically, each Python function is translated to an AST using a given set $\mathcal{R}$ of rules. Based on the puzzle, a context-specific distribution over rule probabilities is computed. To facilitate efficient top-down search, the context of a rule is defined to be the rule used by the parent node and the index of the current node among the parent's children. Thus if the parent node was a division binary operator, then the two children would each have different contexts, but if two such divisions were present in the same program, both numerators would share the same context. 

Each puzzle $f$ is represented by a feature vector $\phi(f)$ and each context is represented by a vector $c(p,i)$ where $p$ is the parent rule and $i$ is the child index. Each rule $r \in \mathcal{R}$ is also associated with a feature vector $\rho(r)$. 
The probability distribution over $\mathcal{R}$ is determined based on $\rho(r), \phi(f), c(p, i)$, and the likelihood of a solution $g$ is the product of all rules constituting its AST. Naturally, this scoring mechanism introduces a bias towards shorter programs (i.e., smaller trees), which is desirable as a short solution is easy to inspect. 

\paragraph{\pyline{COPY} rules.}
Solutions often reuse constants or puzzle parameters, for example the constant \pyline{25} or the variable \pyline{s} in example \pyline{f2} in Figure~\ref{fig:examples}. As in prior work \cite{menon2013machine}, for each puzzle, the global rules bank is expanded to include \pyline{COPY} rules for constants and parameters of the examined puzzle.\footnote{When executing a solution, \pyline{COPY} rules are simply the identity function (\pyline{COPY = lambda x: x} in Python).} When composing solutions, this rule can reduce the complexity of the solution by simply learning to copy part of the puzzle rather than having to generate it from scratch. For simplicity, we create copy rules for each of the supported types and assign the probabilities uniformly across all the puzzle's constants of that type. In other words, our models learn when a certain type should be copied from the puzzle, and rank all available constants and parameters of that type the same.  

To solve a new puzzle, we perform a top-down search. Specifically, at each node, we apply a selected model over all rules in $\mathcal{R}$ whose type matches the context, and re-normalize the scores to create a valid probability distribution. The solver enumerates solutions in order of decreasing likelihood until it finds a solution $g$ such that \pyline{f(g())} evaluates to \pyline{True} in time $\le t$, for a maximum number of tries $M$. See Appendix \ref{sec:td_details} for details on the search and rules. Next, we briefly describe our models.

\paragraph{Uniform.}
The first model is a simple uniform rule that assigns the same probability to all rules. The only exception is \pyline{COPY} rules, which have a larger, fixed probability in order to bias the solver towards utilizing this option. As we score programs by their joint probability, this bias effectively favors shorter programs. We use this model to find solutions to the easier problems, satisfied by a simple and short answer, and use these to bootstrap the learning of the parametric models. This model also provides a naive brute force baseline to compare the parametric models with.

The remaining two models have parameters that are fit based on \textit{bootstrapping}. Namely, given previously obtained solutions, we collect all parent-child rule pairs as self-supervision and fit the model's parameters on them. The training size is then the total number of nodes in all the trees among solutions discovered up until that point. We implement two bigram parametric models to predict $\mathbb{P}\bigl(r\mid\rho(r), \phi(f), c(p, i)\bigr)$, where $r$ is a candidate rule to appear in $g$'s tree under $p$ as its $i$'s argument.

\paragraph{Random forest.}
In this model, we represent $f$ as a bag-of-rules $\cup \{r_k \in f\}$. Specifically, $\phi(f)$ is a vector of length $|\mathcal{R}|$ representing the number of occurrences of each rule in $f$. $p$ and $i$ are encoded as a one-hot vector and concatenated to $f$'s representation to construct the input to the model. Given past solution trees, we train the model to predict the index of $r$ out of $|\mathcal{R}|$ given $f,p,i$ examples.

\paragraph{Transformer.} 
Following the recent success of transformer models~\citep{transformer2017,devlin-etal-2019-bert} in encoding source code~\cite[\emph{inter alia}]{feng-etal-2020-codebert,Svyatkovskiy_2020,cubert}, we turn to these encoders for richer representations. We use a RoBERTa-based~\citep{liu2019roberta} Transformer to encode puzzles and rules directly from their code. The probability of a rule $r$ being the $i$'s child of $p$ in $g$ is proportional to the dot product of the deep joint representation of $f,p,i$ and the Transformer encoding $\rho(r)$. We pretrain the Transformer with a masked language model task on Python GitHub repositories~\citep{husain2019codesearchnet}.\footnote{Our pretrained model and tokenizer are available at \url{https://huggingface.co/tals/roberta_python}.} Then, our solver concatenates the Transformer encodings $\phi(f)$ and $\rho(p)$ with a learned embedding for $i$, following by non-linear layers to compute the joint representation. We fine-tune the solver on parent-child rule pairs from previously acquired solutions. See Appendix~\ref{sec:transformer_details} for extended details, and Figure~\ref{fig:transformer} on page \pageref{fig:transformer} for a model diagram. 

\subsection{Autoregressive Language Model solvers}
We experiment with the transfomer-based GPT-3~\citep{Brown2020FewShot} and Codex~\citep{chen2021evaluating} LMs with billions of parameters. Codex was trained over large amounts of publicly available code from GitHub, specifically aimed for coding applications.
We follow the recent strategy of designing a prompt that directs the text generation to our desired task. This approach has shown to be useful in converting natural language descriptions to programming code and guide theorem proving~\citep{polu2020generative}. Unlike our enumerative models that build an AST, LMs generate the solution as a string that is directly evaluated as Python code.  

We consider four different prompts: (a) A \textit{short} zero-shot prompt based solely on the puzzle at hand (illustrated in Figure \ref{fig:gpt_prompt_short}); (b) a \textit{medium} 5-shot prompt that includes the five example puzzles that had been shown to (human) programmers during our user study (Appendix Figures \ref{fig:gpt_prompt_med}-\ref{fig:codex_prompt_med}); (c) a \textit{long} prompt with the same five puzzles augmented by English descriptions of the tasks in comments (Figures \ref{fig:gpt_prompt_long}-\ref{fig:codex_prompt_long}); and (d) a \textit{bootstrapping} prompt which uses only solutions to problems that it has already solved (Figures \ref{fig:prompt_bootstrap}). The bootstrapping prompt begins with no solutions but quickly exceeds the API maximum length as more puzzles are solved. At that point, previously solved puzzles are randomly sampled to form the prompt. 
The prompts used for Codex are slightly more sophisticated but enable multi-line programs.

The completions which parse as valid Python expressions are then evaluated. Appendix \ref{sec:gpt3_details} gives further details of the execution environment, the API parameters and other prompts we investigated.

\begin{figure}
\small
\begin{pyblock}
def f(li: List[int]):
    return len(li) == 10 and li.count(li[3]) == 2
    
assert True == f($\ldots$
\end{pyblock}
    \vspace{-1\baselineskip}
\caption{A Short prompt for a puzzle requesting a list of ten integers where the fourth item occurs exactly twice, with valid completion $\ldots$\pyline{[1,2,3,4,5]*2)}. Appendix \ref{sec:gpt3_details} has Medium/Long prompts.}
\label{fig:gpt_prompt_short}
    \vspace{-1\baselineskip}
\end{figure}



\section{Experiments}\label{sec:experiments}
We use our P3 dataset to evaluate the performance of the solvers from \S\ref{sec:solvers}. 
We assume no access to reference solutions\footnote{With the exception of the Medium and Long prompts that including five Tutorial problems and solutions.} and measure how many puzzles are solved by each solver with up to $k$ tries per puzzle, where each try is a potential solution that is evaluated. For the enumerative solvers, this is equivalent to having a valid solution ranked in the top $k$. For LM solvers, 
we use $\mathbf{pass}@k$ \citep{chen2021evaluating} which is an unbiased estimator of the probability of obtaining a solution within $k$ tries. First, we test the solvers bootstrapping efficacy in leveraging previously solved problems to solve new ones. Then, once solutions to a single instance of certain problems are found, we test whether solvers also succeed on other problem instances (i.e., puzzles originated from the same problem).
In \S\ref{sec:study_res}, we present our user study results that compares human's performance with AI solvers. Finally, in \S\ref{sec:397}, we test whether P3 can distinguish between subtly different variants of Codex, using the larger v0.2 release of the dataset (the current version at the time of publication). 


\paragraph{Learning from past solutions.} 
This first experiment was run on the v0.1 release of P3.\footnote{\url{https://github.com/microsoft/PythonProgrammingPuzzles/tree/v0.1}} We use a single puzzle instance per problem. We first identified the 138 of the 200 v0.1 problems supported by our grammar (see Appendix~\ref{sec:grammar_details}). For the enumerative solvers, we then ran the uniform solver with $k=10^4$ on these 138 problems supported by our, solving 38 of them. The solutions contain a total of 2,475 rules that we use to train the parametric models. In the bootstrapping variant, we repeat the training for 6 cycles, each time adding the new solutions found with $k=10^4$. In the final round, we allow up to $k=10^6$ solving tries (including previous cycles). For comparison to GPT-3/Codex,\footnote{The Codex API was released after this experiment had been run on v0.1 using the enumerative and GPT-3 solvers. Thus, we simply replaced the GPT-3 solver with the Codex solver and re-ran on the same 138 puzzles.}
we use the same 138 problems and start with a zero-shot prompt. As valid solutions are found, they are appended to the prompt as discussed in \S\ref{sec:gpt3_details}. 


Figure~\ref{fig:top_down_cum} shows the total number of puzzles solved by each enumerative solver, with and without the self-training bootstrapping cycles. We report the average results across three runs and present the standard deviation in the graph.
We see that the parametric models quickly improve over the naive uniform search and that the bootstrapping process facilitates solving many new problems. At $k=10^6$, the random forest and Transformer-based enumerative models solved a total of 68 and 76 problems, respectively, which is 28\% and 43\% more than the uniform solver.

The GPT-3 solver also improves by learning from previously found solutions. As Figure~\ref{fig:gpt3_cum} shows, few-shot settings with tutorial examples perform better than zero-shot (Short) and solve new problems. Including natural language descriptions (Long) helps for solving five more puzzles, with up to $10^4$ tries. The best strategy, however, is the bootstrapping one that starts without any reference and adds solutions to the prompt as they are found. Codex, trained on large amounts of code, performs the best (see Figure~\ref{fig:codex_cum}) but does not benefit significantly from bootstrapping.

\paragraph{Generalizing to other problem instances.} 
In the previous experiment, we attempted to solve the \emph{default} single puzzle instance of each problem. Next, we examine whether our solvers can also solve other puzzle instances, originating from the same problems. We collect a set of 700 puzzles that are random instances of 35 problems for which both our bootstrapping enumerative models solved the default puzzle. At $k=10^4$, the random forest and Transformer models solved 75\% and 79\%, respectively. As a reference, the uniform model solves only 62\% of these puzzles. 

\begin{figure}[t]
    \centering
    
         \small
     \begin{subfigure}[b]{0.32\textwidth}
         \centering
         \includegraphics[width=1.03\textwidth]{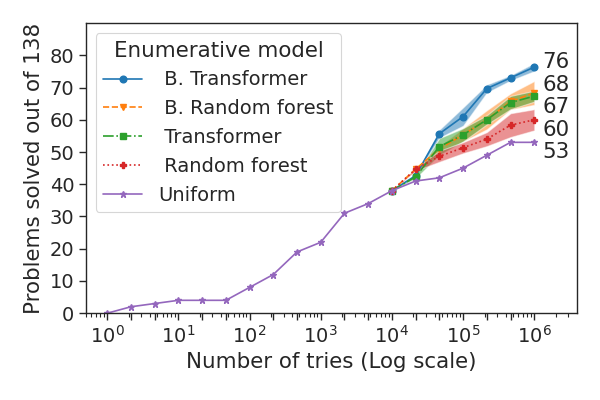}
       \vspace{-1.5\baselineskip}
         \caption{Enumerative solvers}
         \label{fig:top_down_cum}
     \end{subfigure}
     \begin{subfigure}[b]{0.32\textwidth}
         \centering
         \includegraphics[width=1.03\textwidth]{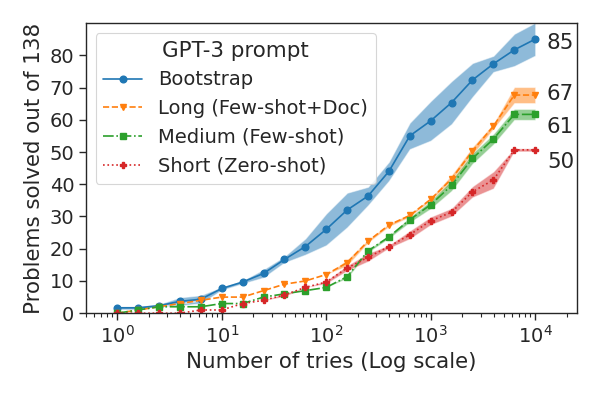}
         \vspace{-1.5\baselineskip}
         \caption{GPT-3 solvers}
         \label{fig:gpt3_cum}
     \end{subfigure}
     \begin{subfigure}[b]{0.32\textwidth}
         \centering
         \includegraphics[width=1.03\textwidth]{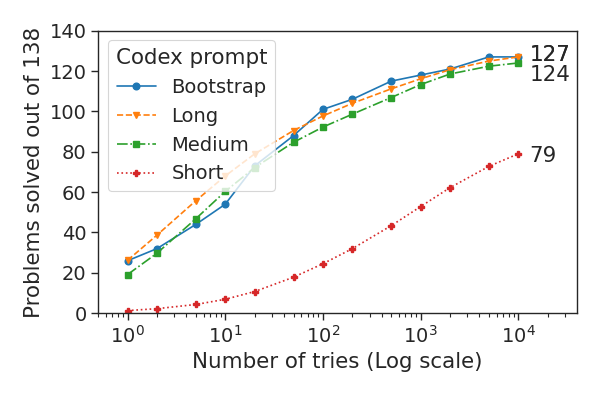}
         \vspace{-1.5\baselineskip}
         \caption{Codex solvers (extended y scale)}
         \label{fig:codex_cum}
     \end{subfigure}
    
    \caption{Increasing the number of tries allows solving new problems. Better solvers, though, solve new problems significantly faster by learning from past experience. Parametric enumerative solvers (a) initialized with the solutions of the uniform solver at $k=10^4$ accelerate the solution search. Additional self-training bootstrapping cycles (marked with B.) solve even more problems. GPT-3 (b) and Codex Davinci (c) solvers were evaluated with up to $10^4$ attempts. Having natural language descriptions (Long) provides small improvements over Medium. Adding previously found solutions to the prompt (Bootstrap) allows significant improvements for enumerative and GPT-3, and matches Long for Codex. Overall, the Codex models performed best, solving up to 127 of the examined 138 puzzles. (a), (b) are averaged across three runs and the shaded areas show the standard deviation.}
    \label{fig:cumsum}
\end{figure}


\subsection{User study}\label{sec:study_res}
In a small user study, 21 participants with varying experience in Python programming attempted to solve 30 puzzles, as found in v0.1 dataset 
as the \pyline{study} module. Each puzzle was allocated a maximum of 6 minutes to solve, and the study was conducted virtually using Jupyter notebooks. Participants were employees at a major software company and were recruited by email and at a hackathon. No compensation was offered. Participants were first given a short tutorial about puzzles and how to submit solutions. The user study files are available in the open-source dataset, and Appendix~\ref{sec:study_details} has further details including the 30 puzzles.


The first finding is that success in puzzles correlates with programming experience.
For our retrospective study analysis, we split the participants by the median years of Python programming experience. We had 10 \emph{beginners} with less than three years of experience, and 11 \emph{experienced} participants with at least three years. We find that 9 of the 30 puzzles were solved by all beginners, while 17 of the puzzles were solved by all experienced participants. 
Also, beginners spent on average 194 seconds per puzzle, while experienced spent only 149 seconds on average. The average solving time provides a useful proxy to the perceived difficulty of each puzzle.
Overall, we see that puzzles are easier for experienced programmers, indicating their value for evaluating programming proficiency.

\begin{wrapfigure}{R}{0.35\textwidth}
    \vspace{-1.4\baselineskip}
    \centering
    \small
    \includegraphics[width=0.35\textwidth]{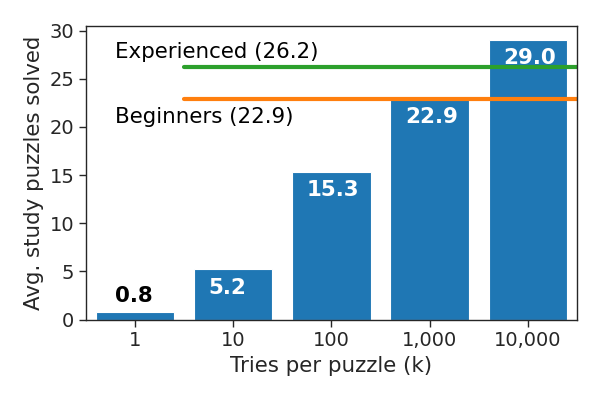}
    \vspace{-1\baselineskip}
    \caption{Number of solved puzzles by Codex-davinci (blue bars), compared to human coders with 6 minutes per puzzle (horizontal lines).}
    \label{fig:codex_study}
    \vspace{-1\baselineskip}
\end{wrapfigure}

Next, we compare human's performance to Codex-davinci. We use the Medium prompt as it is similar to the study format (i.e., same 5 tutorial examples, no docstrings). Participants solved an average of $24.6$ out of the 30 puzzles ($22.9$ for beginners and $26.2$ for experienced) within the 6 minutes per puzzle time limit. Only one out of the 21 participants solved all puzzles. As Figure~\ref{fig:codex_study} shows, Codex required 1K tries per puzzle to match the performance of beginner programmers in our study. 


Finally, we find that difficult puzzles for humans are also harder for AI.
Figure~\ref{fig:user_study_difficulty} shows that most of the puzzles solved by AI solvers with limited number of tries are the ones that are easier for humans (i.e., solved faster). To compare the two, we define a puzzle's perceived difficulty score as the average solving time for humans and the expected number of required tries for machines (normalized to $[0,1]$, where the score of unsolved puzzles is set to $1$). 
The Spearman's rank coefficient of humans with B.~GPT-3 is $0.512$, and with Codex (Med.) is $0.563$. The AI solvers correlation is stronger with beginner programmers ($0.541$ and $0.562$), than with the experienced ones ($0.470$ and $0.544$, respectively).
On the one hand, this suggests that additional computational power might allow AI solvers to match humans. However, as Figure~\ref{fig:cumsum} shows, this improvement is logarithmic, leading to diminishing returns. Encouragingly, we see that even within the same budget, modeling choices can improve performance. We hope that P3 will support the research and development of new AI solvers that will solve more puzzles with less computational effort.


\begin{figure}[t]
    \centering
    \small
    \includegraphics[width=1\textwidth]{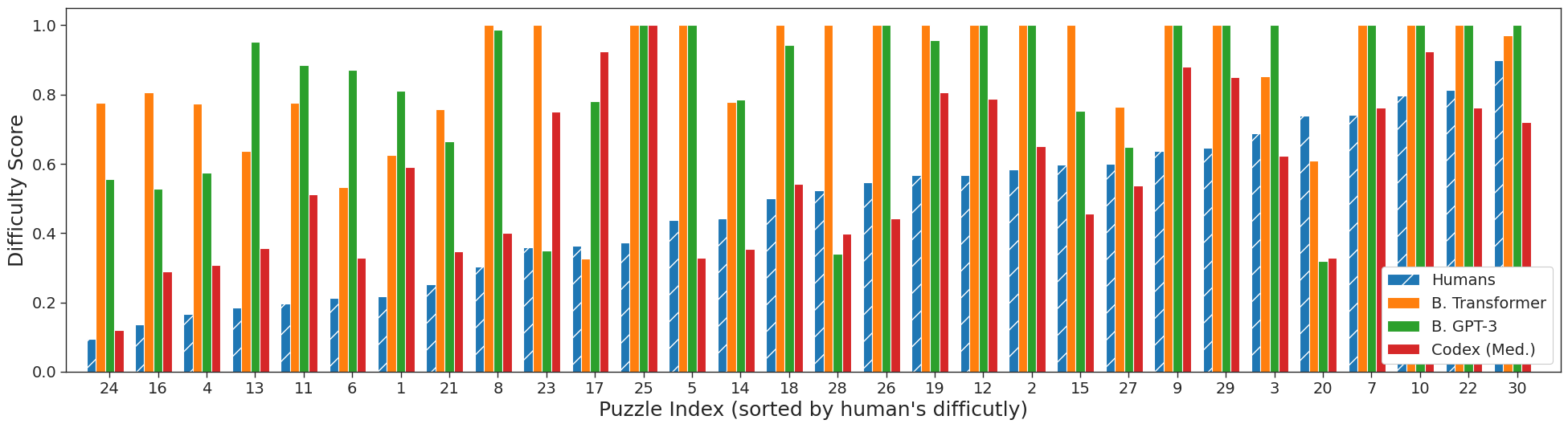}
    \vspace{-1\baselineskip}
    \caption{The difficulty score per study puzzle for both humans and AI solvers, sorted by the human's scores. The difficulty score for humans is measured by the average fraction of solving time out of the maximum allowed. For AI, we use the fraction of allotted attempts required. Most of the puzzles solved by AI (low difficulty score) are also easier for humans (left hand side of the plot).}
    \label{fig:user_study_difficulty}
    \vspace{-1\baselineskip}
\end{figure}

\subsection{Comparing small and large Codex models}\label{sec:397}



\begin{wraptable}{R}{0.47\textwidth}
\vspace{-4mm}
\caption{Codex $\mathbf{pass}@k$ results over P3 v0.2.}
    \resizebox{0.47\columnwidth}{!}{%
    \begin{tabular}{c|zzz}
    \toprule
          engine (prompt)   & \multicolumn{1}{c}{$k=1$} & \multicolumn{1}{c}{$10$} & \multicolumn{1}{c}{$100$} & \multicolumn{1}{c}{1,000} \\
            \midrule
         cushman (Med.) & ~~7.1\% & 26.7\% & 51.7\% & 68.3\% \\
         davinci (Med.) & 11.2\% & 36.7\% & 60.6\% & 75.3\% \\
         \midrule
         cushman (Long) & 14.9\% & 42.4\% & 63.9\% & 76.5\% \\
         davinci (Long) & 18.3\% & 48.7\% & 69.1\% & 79.8\% \\
         \bottomrule
    \end{tabular}
    }%
\label{tab:codex_res}
\end{wraptable}

In addition to the standard \textit{davinci-codex} engine, the API offers an alternate \textit{cushman-codex} engine that they report is significantly faster and only slightly less accurate. To test the ability of P3 as an evaluation of such fine distinctions, we ran the Medium and Long prompts on both engines across the most recent v0.2 release\footnote{\url{https://github.com/microsoft/PythonProgrammingPuzzles/tree/v0.2}} of 397 puzzles. As can be seen in the results of Table \ref{tab:codex_res}, the larger engine indeed slightly outperformed the smaller engine across all $k$. Thus, in this experiment, puzzle solving success aligns with code completion success. Also, we observe that English descriptions (Long prompt) are helpful for both engines. Inasmuch as puzzles are useful for code completion, the $< 20\%$ success rates at $k=1$ leaves substantial room for improvement.

Table~\ref{tab:top_down_cat_breakdown} shows the number of achieved solutions per domain, as well as an overall score computed as the macro-average of solving rates across domains.

\begin{table}[t]
    \centering
    \small
     \caption{Codex-davinci and Codex-cushman number of solved problems per domain with up to 1,000 tries for Medium and Long prompts.  The first row also shows the number of available P3 v0.2 problems in that domain. The score is the average percent solved across domains.}
      \begin{subtable}{1\columnwidth}
    \resizebox{1\columnwidth}{!}{%
    \begin{tabular}{c|sssssssss}
    \toprule
         Model & \multicolumn{1}{c}{Algebra} & \multicolumn{1}{c}{Basic} & \multicolumn{1}{c}{Chess} & \multicolumn{1}{c}{Classic} & \multicolumn{1}{c}{CodeForces} & \multicolumn{1}{c}{Compression} & \multicolumn{1}{c}{Conway's} & \multicolumn{1}{c}{Games} & \multicolumn{1}{c}{Graphs} \\
         \midrule
cushman (Med.) & 3/4 & 15/23 & 0/5 & 6/23 & 32/47 & 0/2 & 0/3 & 1/7 & 6/12 \\
\rowcolor{Gray} davinci (Med.) & 2 & 20 & 0 & 8 & 35 & 0 & 0 & 1 & 9 \\
cushman (Long) & 2 & 21 & 0 & 5 & 38 & 0 & 0 & 1 & 8 \\
\rowcolor{Gray} davinci (Long) & 2 & 22 & 1 & 4 & 39 & 0 & 0 & 1 & 8 \\
    \bottomrule
    \end{tabular}
    }%
    \vspace*{0.2\baselineskip}
    \newline
        \resizebox{1\columnwidth}{!}{%
    \begin{tabular}{c|sssssssss|d}
    \toprule
         Model & \multicolumn{1}{c}{HumanEval} & \multicolumn{1}{c}{ICPC} & \multicolumn{1}{c}{IMO} & \multicolumn{1}{c}{Lattices} & \multicolumn{1}{c}{N. Theory} & \multicolumn{1}{c}{Probability} & \multicolumn{1}{c}{Study} & \multicolumn{1}{c}{Trivial$^{-1}$} & \multicolumn{1}{c}{Tutorial} & \multicolumn{1}{|c}{Score} \\
         \midrule
cushman (Med.) & 139/164 & 2/4 & 1/6 & 0/2 & 8/16 & 2/5 & 21/30 & 33/39 & 5/5 & 44.2 \\
\rowcolor{Gray} davinci (Med.) & 145 & 2 & 1 & 1 & 9 & 3 & 22 & 36 & 5 & 51.2 \\
cushman (Long) & 149 & 1 & 1 & 1 & 9 & 3 & 24 & 36 & 5 & 49.8 \\
\rowcolor{Gray} davinci (Long) & 155 & 1 & 1 & 2 & 10 & 3 & 25 & 38 & 5 & 54.8 \\
    \bottomrule
    \end{tabular}
        }%
    \end{subtable}
   
    \label{tab:top_down_cat_breakdown}
     \vspace{-1\baselineskip}
\end{table}
\section{Related Work}\label{sec:related}

Program synthesis has taken drastically different forms for different applications, often resulting in one-off evaluations rather than common datasets. A major paradigm is Programming by Example (PbE) where problems are specified by input-output examples. For instance, several studies focus on text processing \citep{prose_automating} or robot navigation~\citep{karel_the_robot}. While convenient for end user applications (e.g., many in~\citep{polozov2015flashmeta}), PbE alone is inadequate to objectively describe many sophisticated algorithmic programming challenges.
A recent ARC dataset~\cite{chollet2019measure} adopts PbE for evaluating abstraction and reasoning in AI, but as in all PbE applications, there can be ambiguity.

Program synthesis from formal specifications has a long history of study \citep[surveyed in][]{synthesisSurvey2017}, benchmarked by e.g., the SyGuS competition~\citep{alur2019syguscomp}. In this setting, however, the AI system has to synthesize an algorithm that correctly and efficiently solves a problem on all inputs (and often prove correctness as well). Writing and testing such formal specifications is often non-trivial. 

English descriptions, often mixed with examples, are becoming an increasingly popular problem representation as LMs improve \citep{kulal_spoc, polosukhin2018neural, naps2018}. In independent work, \citet{hendrycks2021measuring} created a large dataset of English programming problems with examples on which they fine-tuned GPT models. In another concurrent work, the Codex model that powers the new GitHub Copilot auto-completion tool~\citep{chen2021evaluating} was evaluated with short problem descriptions paired with a set of unit tests that should validate the described specification.
Our work, together with this very recent and concurrent work \citep{chen2021evaluating, austin2021program, hendrycks2021measuring}, represent the first controlled evaluation of large Transformer-based LMs on general-purpose program synthesis. 

The recent CodeXGLUE benchmark~\citep{lu2021codexglue} collected several code-related datasets. To evaluate generation, they use CodeBLEU~\citep{ren2020codebleu} which relies on ASTs and other code-specific aspects. This evaluation still requires reference solutions and, therefore, does not resolve the answer-key bias with ambiguous specifications. Several neighboring fields that have made substantial progress in reasoning include theorem proving~\cite{coq}, two-player game playing~\cite{alphaZero}, and SAT-solving~\cite{satbook}. In all these fields, important progress has been made by encoding the problems, be they theorems, game rules, or optimization problems, in machine-readable formats that do not involve the ambiguities of natural language.

\section{Conclusions}
We introduce Python Programming Puzzles (P3), an open-source dataset with puzzles described only in source code. 
As discussed in \S\ref{sec:the_dataset}, the puzzle framework captures NP problems, which include a wide range of interesting challenges. Puzzles allow fast and objective evaluation, thereby supporting unsupervised solving without training solutions.
We implemented and evaluated several enumerative program-synthesis and LM baselines, and found a positive correlation between their per-puzzle performance and the difficulty for human programmers. Similarly, LMs that performed better at code completion also solved more puzzles with less tries.

We welcome contributions to P3 and hope it will grow in size, coverage, and utility.

\paragraph{Acknowledgments.} We would like to thank Mariia Mykhailova for suggesting doing a Python Programming Puzzles Hackathon. We are especially grateful to the participants in our user study and hackathon. We are grateful to the creators of Codex and GPT-3 and to Nicol\'{o} Fusi for suggesting its use in this project. We would like to thank David Alvarez Melis and Alec Helbing for suggesting quine puzzles. We are grateful to Ana-Roxana Pop for helpful discussions and feedback. We also thank Tianxiao Shen, Adam Fisch and the rest of the MIT NLP members for valuable writing feedback.







\bibliographystyle{plainnat}
\bibliography{bibliography}

\begin{thebibliography}{63}
\providecommand{\natexlab}[1]{#1}
\providecommand{\url}[1]{\texttt{#1}}
\expandafter\ifx\csname urlstyle\endcsname\relax
  \providecommand{\doi}[1]{doi: #1}\else
  \providecommand{\doi}{doi: \begingroup \urlstyle{rm}\Url}\fi

\bibitem[Alon et~al.(2020)Alon, Sadaka, Levy, and Yahav]{alon2020structural}
Uri Alon, Roy Sadaka, Omer Levy, and Eran Yahav.
\newblock Structural language models of code.
\newblock In \emph{International Conference on Machine Learning}, pages
  245--256. PMLR, 2020.

\bibitem[Alur et~al.(2019)Alur, Fisman, Padhi, Singh, and
  Udupa]{alur2019syguscomp}
Rajeev Alur, Dana Fisman, Saswat Padhi, Rishabh Singh, and Abhishek Udupa.
\newblock {SyGuS-Comp} 2018: Results and analysis.
\newblock 2019.

\bibitem[Andreas et~al.(2020)Andreas, Bufe, Burkett, Chen, Clausman, Crawford,
  Crim, DeLoach, Dorner, Eisner, Fang, Guo, Hall, Hayes, Hill, Ho, Iwaszuk,
  Jha, Klein, Krishnamurthy, Lanman, Liang, Lin, Lintsbakh, McGovern,
  Nisnevich, Pauls, Petters, Read, Roth, Roy, Rusak, Short, Slomin, Snyder,
  Striplin, Su, Tellman, Thomson, Vorobev, Witoszko, Wolfe, Wray, Zhang, and
  Zotov]{andreas-etal-2020-task}
Jacob Andreas, John Bufe, David Burkett, Charles Chen, Josh Clausman, Jean
  Crawford, Kate Crim, Jordan DeLoach, Leah Dorner, Jason Eisner, Hao Fang,
  Alan Guo, David Hall, Kristin Hayes, Kellie Hill, Diana Ho, Wendy Iwaszuk,
  Smriti Jha, Dan Klein, Jayant Krishnamurthy, Theo Lanman, Percy Liang,
  Christopher~H. Lin, Ilya Lintsbakh, Andy McGovern, Aleksandr Nisnevich, Adam
  Pauls, Dmitrij Petters, Brent Read, Dan Roth, Subhro Roy, Jesse Rusak, Beth
  Short, Div Slomin, Ben Snyder, Stephon Striplin, Yu~Su, Zachary Tellman, Sam
  Thomson, Andrei Vorobev, Izabela Witoszko, Jason Wolfe, Abby Wray, Yuchen
  Zhang, and Alexander Zotov.
\newblock Task-oriented dialogue as dataflow synthesis.
\newblock \emph{Transactions of the Association for Computational Linguistics},
  8:\penalty0 556--571, 2020.
\newblock \doi{10.1162/tacl_a_00333}.
\newblock URL \url{https://aclanthology.org/2020.tacl-1.36}.

\bibitem[Arora and Barak(2009)]{Arora2009ComputationalCA}
Sanjeev Arora and B.~Barak.
\newblock Computational complexity: A modern approach.
\newblock 2009.

\bibitem[Austin et~al.(2021)Austin, Odena, Nye, Bosma, Michalewski, Dohan,
  Jiang, Cai, Terry, Le, and Sutton]{austin2021program}
Jacob Austin, Augustus Odena, Maxwell Nye, Maarten Bosma, Henryk Michalewski,
  David Dohan, Ellen Jiang, Carrie Cai, Michael Terry, Quoc Le, and Charles
  Sutton.
\newblock Program synthesis with large language models, 2021.

\bibitem[Balog et~al.(2017)Balog, Gaunt, Brockschmidt, Nowozin, and
  Tarlow]{balog2016deepcoder}
Matej Balog, Alexander~L. Gaunt, Marc Brockschmidt, Sebastian Nowozin, and
  Daniel Tarlow.
\newblock Deepcoder: Learning to write programs.
\newblock In \emph{International Conference on Representation Learning (ICLR)},
  2017.

\bibitem[Berlekamp et~al.(1978)Berlekamp, McEliece, and van
  Tilborg]{berlekamp1978}
E.~Berlekamp, R.~McEliece, and H.~van Tilborg.
\newblock On the inherent intractability of certain coding problems (corresp.).
\newblock \emph{IEEE Transactions on Information Theory}, 24\penalty0
  (3):\penalty0 384--386, 1978.
\newblock \doi{10.1109/TIT.1978.1055873}.

\bibitem[Bertot and Cast{\'e}ran(2004)]{coq}
Yves Bertot and Pierre Cast{\'e}ran.
\newblock \emph{Interactive Theorem Proving and Program Development: Coq’Art:
  The Calculus of Inductive Constructions}.
\newblock Springer Science \& Business Media, 2004.

\bibitem[Biere et~al.(2009)Biere, Heule, and van Maaren]{satbook}
Armin Biere, Marijn Heule, and Hans van Maaren.
\newblock \emph{Handbook of satisfiability}, volume 185.
\newblock IOS press, 2009.

\bibitem[Biggs et~al.(1971)Biggs, Biggs, Society, lecture~note series, and
  Hitchin]{biggs1971finite}
N.~Biggs, P.M.L.S.E.N.L. Biggs, London~Mathematical Society, London
  Mathematical~Society lecture~note series, and S.P.G.N.J. Hitchin.
\newblock \emph{Finite Groups of Automorphisms: Course Given at the University
  of Southampton, October-December 1969}.
\newblock Lecture note series. Cambridge University Press, 1971.
\newblock ISBN 9780521082150.
\newblock URL \url{https://books.google.com/books?id=flA4AAAAIAAJ}.

\bibitem[Blum et~al.(2003)Blum, Kalai, and Wasserman]{blum2003}
Avrim Blum, Adam Kalai, and Hal Wasserman.
\newblock Noise-tolerant learning, the parity problem, and the statistical
  query model.
\newblock \emph{J. ACM}, 50\penalty0 (4):\penalty0 506–519, July 2003.
\newblock ISSN 0004-5411.
\newblock \doi{10.1145/792538.792543}.
\newblock URL \url{https://doi.org/10.1145/792538.792543}.

\bibitem[Brown et~al.(2020)Brown, Mann, Ryder, Subbiah, Kaplan, Dhariwal,
  Neelakantan, Shyam, Sastry, Askell, Agarwal, Herbert-Voss, Krueger, Henighan,
  Child, Ramesh, Ziegler, Wu, Winter, Hesse, Chen, Sigler, Litwin, Gray, Chess,
  Clark, Berner, McCandlish, Radford, Sutskever, and Amodei]{Brown2020FewShot}
Tom Brown, Benjamin Mann, Nick Ryder, Melanie Subbiah, Jared~D Kaplan, Prafulla
  Dhariwal, Arvind Neelakantan, Pranav Shyam, Girish Sastry, Amanda Askell,
  Sandhini Agarwal, Ariel Herbert-Voss, Gretchen Krueger, Tom Henighan, Rewon
  Child, Aditya Ramesh, Daniel Ziegler, Jeffrey Wu, Clemens Winter, Chris
  Hesse, Mark Chen, Eric Sigler, Mateusz Litwin, Scott Gray, Benjamin Chess,
  Jack Clark, Christopher Berner, Sam McCandlish, Alec Radford, Ilya Sutskever,
  and Dario Amodei.
\newblock Language models are few-shot learners.
\newblock In H.~Larochelle, M.~Ranzato, R.~Hadsell, M.~F. Balcan, and H.~Lin,
  editors, \emph{Advances in Neural Information Processing Systems}, volume~33,
  pages 1877--1901. Curran Associates, Inc., 2020.
\newblock URL
  \url{https://proceedings.neurips.cc/paper/2020/file/1457c0d6bfcb4967418bfb8ac142f64a-Paper.pdf}.

\bibitem[Chen et~al.(2021)Chen, Tworek, Jun, Yuan, Ponde, Kaplan, Edwards,
  Burda, Joseph, Brockman, Ray, Puri, Krueger, Petrov, Khlaaf, Sastry, Mishkin,
  Chan, Gray, Ryder, Pavlov, Power, Kaiser, Bavarian, Winter, Tillet, Such,
  Cummings, Plappert, Chantzis, Barnes, Herbert-Voss, Guss, Nichol, Babuschkin,
  Balaji, Jain, Carr, Leike, Achiam, Misra, Morikawa, Radford, Knight,
  Brundage, Murati, Mayer, Welinder, McGrew, Amodei, McCandlish, Sutskever, and
  Zaremba]{chen2021evaluating}
Mark Chen, Jerry Tworek, Heewoo Jun, Qiming Yuan, Henrique Ponde, Jared Kaplan,
  Harri Edwards, Yura Burda, Nicholas Joseph, Greg Brockman, Alex Ray, Raul
  Puri, Gretchen Krueger, Michael Petrov, Heidy Khlaaf, Girish Sastry, Pamela
  Mishkin, Brooke Chan, Scott Gray, Nick Ryder, Mikhail Pavlov, Alethea Power,
  Lukasz Kaiser, Mohammad Bavarian, Clemens Winter, Philippe Tillet, Felipe
  Such, Dave Cummings, Matthias Plappert, Fotios Chantzis, Elizabeth Barnes,
  Ariel Herbert-Voss, Will Guss, Alex Nichol, Igor Babuschkin, Suchir Balaji,
  Shantanu Jain, Andrew Carr, Jan Leike, Josh Achiam, Vedant Misra, Evan
  Morikawa, Alec Radford, Matthew Knight, Miles Brundage, Mira Murati, Katie
  Mayer, Peter Welinder, Bob McGrew, Dario Amodei, Sam McCandlish, Ilya
  Sutskever, and Wojciech Zaremba.
\newblock Evaluating large language models trained on code, 2021.

\bibitem[Chollet(2019)]{chollet2019measure}
Fran{\c{c}}ois Chollet.
\newblock On the measure of intelligence.
\newblock \emph{arXiv preprint arXiv:1911.01547}, 2019.

\bibitem[Christakopoulou and Kalai(2018)]{christakopoulou2018glass}
Konstantina Christakopoulou and Adam~Tauman Kalai.
\newblock Glass-box program synthesis: A machine learning approach.
\newblock In \emph{Thirty-Second AAAI Conference on Artificial Intelligence},
  2018.

\bibitem[Conway(2017)]{Conway99}
John~Horton Conway.
\newblock Five \$1,000 problems (update 2017).
\newblock 2017.
\newblock URL \url{https://oeis.org/A248380/a248380.pdf}.
\newblock [Online; accessed 12/15/2020].

\bibitem[Dagien{\.e} and Futschek(2008)]{dagiene2008bebras}
Valentina Dagien{\.e} and Gerald Futschek.
\newblock Bebras international contest on informatics and computer literacy:
  Criteria for good tasks.
\newblock In \emph{International conference on informatics in secondary
  schools-evolution and perspectives}, pages 19--30. Springer, 2008.

\bibitem[Dantzig(1951)]{dantzig1951proof}
George~B Dantzig.
\newblock A proof of the equivalence of the programming problem and the game
  problem.
\newblock \emph{Activity analysis of production and allocation}, 13:\penalty0
  330--338, 1951.

\bibitem[Dechter et~al.(2013)Dechter, Malmaud, Adams, and
  Tenenbaum]{dechter2013bootstrap}
Eyal Dechter, Jonathan Malmaud, Ryan~P Adams, and Joshua~B Tenenbaum.
\newblock Bootstrap learning via modular concept discovery.
\newblock In \emph{IJCAI}, pages 1302--1309, 2013.

\bibitem[Devlin et~al.(2019)Devlin, Chang, Lee, and
  Toutanova]{devlin-etal-2019-bert}
Jacob Devlin, Ming-Wei Chang, Kenton Lee, and Kristina Toutanova.
\newblock {BERT}: Pre-training of deep bidirectional transformers for language
  understanding.
\newblock In \emph{Proceedings of the 2019 Conference of the North {A}merican
  Chapter of the Association for Computational Linguistics: Human Language
  Technologies, Volume 1 (Long and Short Papers)}, pages 4171--4186,
  Minneapolis, Minnesota, June 2019. Association for Computational Linguistics.
\newblock \doi{10.18653/v1/N19-1423}.
\newblock URL \url{https://www.aclweb.org/anthology/N19-1423}.

\bibitem[Feng et~al.(2020)Feng, Guo, Tang, Duan, Feng, Gong, Shou, Qin, Liu,
  Jiang, and Zhou]{feng-etal-2020-codebert}
Zhangyin Feng, Daya Guo, Duyu Tang, Nan Duan, Xiaocheng Feng, Ming Gong, Linjun
  Shou, Bing Qin, Ting Liu, Daxin Jiang, and Ming Zhou.
\newblock {C}ode{BERT}: A pre-trained model for programming and natural
  languages.
\newblock In \emph{Findings of the Association for Computational Linguistics:
  EMNLP 2020}, pages 1536--1547, Online, November 2020. Association for
  Computational Linguistics.
\newblock \doi{10.18653/v1/2020.findings-emnlp.139}.
\newblock URL \url{https://www.aclweb.org/anthology/2020.findings-emnlp.139}.

\bibitem[Gulwani(2011)]{prose_automating}
Sumit Gulwani.
\newblock Automating string processing in spreadsheets using input-output
  examples.
\newblock In \emph{Proceedings of the 38th Annual ACM SIGPLAN-SIGACT Symposium
  on Principles of Programming Languages}, POPL '11, page 317–330, New York,
  NY, USA, 2011. Association for Computing Machinery.
\newblock ISBN 9781450304900.
\newblock \doi{10.1145/1926385.1926423}.
\newblock URL \url{https://doi.org/10.1145/1926385.1926423}.

\bibitem[Gulwani et~al.(2017)Gulwani, Polozov, Singh,
  et~al.]{synthesisSurvey2017}
Sumit Gulwani, Oleksandr Polozov, Rishabh Singh, et~al.
\newblock Program synthesis.
\newblock \emph{Foundations and Trends{\textregistered} in Programming
  Languages}, 4\penalty0 (1-2):\penalty0 1--119, 2017.

\bibitem[Hardy et~al.(2008)Hardy, Wright, Heath-Brown, and
  Silverman]{theory_of_numbers_book}
G.~H. Hardy, E.~M. Wright, D.~R. Heath-Brown, and Joseph~H. Silverman.
\newblock \emph{An introduction to the theory of numbers}.
\newblock Oxford University Press, Oxford; New York, 2008.
\newblock ISBN 9780199219858 0199219850 9780199219865 0199219869.

\bibitem[Hendrycks and Gimpel(2016)]{hendrycks2016gaussian}
Dan Hendrycks and Kevin Gimpel.
\newblock Gaussian error linear units (gelus).
\newblock 2016.

\bibitem[Hendrycks et~al.(2021)Hendrycks, Basart, Kadavath, Mazeika, Arora,
  Guo, Burns, Puranik, He, Song, and Steinhardt]{hendrycks2021measuring}
Dan Hendrycks, Steven Basart, Saurav Kadavath, Mantas Mazeika, Akul Arora,
  Ethan Guo, Collin Burns, Samir Puranik, Horace He, Dawn Song, and Jacob
  Steinhardt.
\newblock Measuring coding challenge competence with {APPS}.
\newblock 2021.

\bibitem[Husain et~al.(2019)Husain, Wu, Gazit, Allamanis, and
  Brockschmidt]{husain2019codesearchnet}
Hamel Husain, Ho-Hsiang Wu, Tiferet Gazit, Miltiadis Allamanis, and Marc
  Brockschmidt.
\newblock {CodeSearchNet} challenge: Evaluating the state of semantic code
  search.
\newblock \emph{arXiv preprint arXiv:1909.09436}, 2019.

\bibitem[Illia~Polosukhin(2018)]{polosukhin2018neural}
Alexander~Skidanov Illia~Polosukhin.
\newblock Neural program search: Solving data processing tasks from description
  and examples.
\newblock In \emph{ICLR Workshop Acceptance Decision}, 2018.
\newblock URL \url{https://openreview.net/forum?id=B1KJJf-R-}.

\bibitem[Ioffe and Szegedy(2015)]{ioffe2015batch}
Sergey Ioffe and Christian Szegedy.
\newblock Batch normalization: Accelerating deep network training by reducing
  internal covariate shift.
\newblock 2015.

\bibitem[Kaliski(2005)]{Kaliski2005}
Burt Kaliski.
\newblock \emph{{RSA} factoring challenge}, pages 531--532.
\newblock Springer US, Boston, MA, 2005.
\newblock ISBN 978-0-387-23483-0.
\newblock \doi{10.1007/0-387-23483-7_362}.
\newblock URL \url{https://doi.org/10.1007/0-387-23483-7_362}.

\bibitem[Kanade et~al.(2020)Kanade, Maniatis, Balakrishnan, and Shi]{cubert}
Aditya Kanade, Petros Maniatis, Gogul Balakrishnan, and Kensen Shi.
\newblock Learning and evaluating contextual embedding of source code.
\newblock In \emph{Proceedings of the 37th International Conference on Machine
  Learning, {ICML} 2020, 12-18 July 2020}, Proceedings of Machine Learning
  Research. {PMLR}, 2020.

\bibitem[Kiela et~al.(2021)Kiela, Bartolo, Nie, Kaushik, Geiger, Wu, Vidgen,
  Prasad, Singh, Ringshia, Ma, Thrush, Riedel, Waseem, Stenetorp, Jia, Bansal,
  Potts, and Williams]{kiela-etal-2021-dynabench}
Douwe Kiela, Max Bartolo, Yixin Nie, Divyansh Kaushik, Atticus Geiger,
  Zhengxuan Wu, Bertie Vidgen, Grusha Prasad, Amanpreet Singh, Pratik Ringshia,
  Zhiyi Ma, Tristan Thrush, Sebastian Riedel, Zeerak Waseem, Pontus Stenetorp,
  Robin Jia, Mohit Bansal, Christopher Potts, and Adina Williams.
\newblock Dynabench: Rethinking benchmarking in {NLP}.
\newblock In \emph{Proceedings of the 2021 Conference of the North American
  Chapter of the Association for Computational Linguistics: Human Language
  Technologies}, pages 4110--4124, Online, June 2021. Association for
  Computational Linguistics.
\newblock \doi{10.18653/v1/2021.naacl-main.324}.
\newblock URL \url{https://aclanthology.org/2021.naacl-main.324}.

\bibitem[Kulal et~al.(2019)Kulal, Pasupat, Chandra, Lee, Padon, Aiken, and
  Liang]{kulal_spoc}
Sumith Kulal, Panupong Pasupat, Kartik Chandra, Mina Lee, Oded Padon, Alex
  Aiken, and Percy~S Liang.
\newblock {SPoC}: Search-based pseudocode to code.
\newblock In H.~Wallach, H.~Larochelle, A.~Beygelzimer, F.~d\textquotesingle
  Alch\'{e}-Buc, E.~Fox, and R.~Garnett, editors, \emph{Advances in Neural
  Information Processing Systems}, volume~32. Curran Associates, Inc., 2019.
\newblock URL
  \url{https://proceedings.neurips.cc/paper/2019/file/7298332f04ac004a0ca44cc69ecf6f6b-Paper.pdf}.

\bibitem[Laaksonen(2020)]{laaksonen2020guide}
A.~Laaksonen.
\newblock \emph{Guide to Competitive Programming: Learning and Improving
  Algorithms Through Contests}.
\newblock Undergraduate Topics in Computer Science. Springer International
  Publishing, 2020.
\newblock ISBN 9783030393571.
\newblock URL \url{https://books.google.com/books?id=3JbiDwAAQBAJ}.

\bibitem[Lagarias(1985)]{collatzConjecture}
Jeffrey~C. Lagarias.
\newblock The 3x + 1 problem and its generalizations.
\newblock \emph{The American Mathematical Monthly}, 92\penalty0 (1):\penalty0
  3--23, 1985.
\newblock ISSN 00029890, 19300972.
\newblock URL \url{http://www.jstor.org/stable/2322189}.

\bibitem[Liu et~al.(2019)Liu, Ott, Goyal, Du, Joshi, Chen, Levy, Lewis,
  Zettlemoyer, and Stoyanov]{liu2019roberta}
Yinhan Liu, Myle Ott, Naman Goyal, Jingfei Du, Mandar Joshi, Danqi Chen, Omer
  Levy, Mike Lewis, Luke Zettlemoyer, and Veselin Stoyanov.
\newblock Roberta: A robustly optimized bert pretraining approach.
\newblock 2019.

\bibitem[Lourie et~al.(2021)Lourie, {Le Bras}, Bhagavatula, and
  Choi]{Lourie2021UNICORNOR}
Nicholas Lourie, Ronan {Le Bras}, Chandra Bhagavatula, and Yejin Choi.
\newblock Unicorn on rainbow: A universal commonsense reasoning model on a new
  multitask benchmark.
\newblock \emph{AAAI}, 2021.

\bibitem[Lu et~al.(2021)Lu, Guo, Ren, Huang, Svyatkovskiy, Blanco, Clement,
  Drain, Jiang, Tang, Li, Zhou, Shou, Zhou, Tufano, Gong, Zhou, Duan,
  Sundaresan, Deng, Fu, and Liu]{lu2021codexglue}
Shuai Lu, Daya Guo, Shuo Ren, Junjie Huang, Alexey Svyatkovskiy, Ambrosio
  Blanco, Colin Clement, Dawn Drain, Daxin Jiang, Duyu Tang, Ge~Li, Lidong
  Zhou, Linjun Shou, Long Zhou, Michele Tufano, Ming Gong, Ming Zhou, Nan Duan,
  Neel Sundaresan, Shao~Kun Deng, Shengyu Fu, and Shujie Liu.
\newblock Codexglue: A machine learning benchmark dataset for code
  understanding and generation, 2021.

\bibitem[Mayer(2020)]{pythonOneLiners}
C.~Mayer.
\newblock \emph{Python One-Liners: Write Concise, Eloquent Python Like a
  Professional}.
\newblock No Starch Press, Incorporated, 2020.
\newblock ISBN 9781718500501.
\newblock URL \url{https://books.google.com/books?id=jVv6DwAAQBAJ}.

\bibitem[Menon et~al.(2013)Menon, Tamuz, Gulwani, Lampson, and
  Kalai]{menon2013machine}
Aditya~Krishna Menon, Omer Tamuz, Sumit Gulwani, Butler~W Lampson, and Adam
  Kalai.
\newblock A machine learning framework for programming by example.
\newblock In \emph{ICML}, pages 187--195, 2013.

\bibitem[Pedregosa et~al.(2011)Pedregosa, Varoquaux, Gramfort, Michel, Thirion,
  Grisel, Blondel, Prettenhofer, Weiss, Dubourg, Vanderplas, Passos,
  Cournapeau, Brucher, Perrot, and Duchesnay]{scikit-learn}
F.~Pedregosa, G.~Varoquaux, A.~Gramfort, V.~Michel, B.~Thirion, O.~Grisel,
  M.~Blondel, P.~Prettenhofer, R.~Weiss, V.~Dubourg, J.~Vanderplas, A.~Passos,
  D.~Cournapeau, M.~Brucher, M.~Perrot, and E.~Duchesnay.
\newblock Scikit-learn: Machine learning in {P}ython.
\newblock \emph{Journal of Machine Learning Research}, 12:\penalty0 2825--2830,
  2011.

\bibitem[Petroni et~al.(2019)Petroni, Rockt{\"a}schel, Riedel, Lewis, Bakhtin,
  Wu, and Miller]{petroni-etal-2019-language}
Fabio Petroni, Tim Rockt{\"a}schel, Sebastian Riedel, Patrick Lewis, Anton
  Bakhtin, Yuxiang Wu, and Alexander Miller.
\newblock Language models as knowledge bases?
\newblock In \emph{Proceedings of the 2019 Conference on Empirical Methods in
  Natural Language Processing and the 9th International Joint Conference on
  Natural Language Processing (EMNLP-IJCNLP)}, pages 2463--2473, Hong Kong,
  China, November 2019. Association for Computational Linguistics.
\newblock \doi{10.18653/v1/D19-1250}.
\newblock URL \url{https://aclanthology.org/D19-1250}.

\bibitem[Piech and Roberts(2020)]{karel_the_robot}
Chris Piech and Eric Roberts.
\newblock Karel the robot learns python.
\newblock
  \url{https://compedu.stanford.edu/karel-reader/docs/python/en/intro.html},
  2020.
\newblock [Online; accessed 07-Jun-2021].

\bibitem[Polozov and Gulwani(2015)]{polozov2015flashmeta}
Oleksandr Polozov and Sumit Gulwani.
\newblock {FlashMeta}: A framework for inductive program synthesis.
\newblock In \emph{Proceedings of the 2015 ACM SIGPLAN International Conference
  on Object-Oriented Programming, Systems, Languages, and Applications}, pages
  107--126, 2015.

\bibitem[Polu and Sutskever(2020)]{polu2020generative}
Stanislas Polu and Ilya Sutskever.
\newblock Generative language modeling for automated theorem proving.
\newblock \emph{CoRR}, abs/2009.03393, 2020.
\newblock URL \url{https://arxiv.org/abs/2009.03393}.

\bibitem[Rajpurkar et~al.(2016)Rajpurkar, Zhang, Lopyrev, and
  Liang]{rajpurkar-etal-2016-squad}
Pranav Rajpurkar, Jian Zhang, Konstantin Lopyrev, and Percy Liang.
\newblock {SQ}u{AD}: 100,000+ questions for machine comprehension of text.
\newblock In \emph{Proceedings of the 2016 Conference on Empirical Methods in
  Natural Language Processing}, pages 2383--2392, Austin, Texas, November 2016.
  Association for Computational Linguistics.
\newblock \doi{10.18653/v1/D16-1264}.
\newblock URL \url{https://aclanthology.org/D16-1264}.

\bibitem[Ren et~al.(2020)Ren, Guo, Lu, Zhou, Liu, Tang, Sundaresan, Zhou,
  Blanco, and Ma]{ren2020codebleu}
Shuo Ren, Daya Guo, Shuai Lu, Long Zhou, Shujie Liu, Duyu Tang, Neel
  Sundaresan, Ming Zhou, Ambrosio Blanco, and Shuai Ma.
\newblock Codebleu: a method for automatic evaluation of code synthesis, 2020.

\bibitem[Russakovsky et~al.(2015)Russakovsky, Deng, Su, Krause, Satheesh, Ma,
  Huang, Karpathy, Khosla, Bernstein, Berg, and Fei-Fei]{ILSVRC15}
Olga Russakovsky, Jia Deng, Hao Su, Jonathan Krause, Sanjeev Satheesh, Sean Ma,
  Zhiheng Huang, Andrej Karpathy, Aditya Khosla, Michael Bernstein,
  Alexander~C. Berg, and Li~Fei-Fei.
\newblock {ImageNet Large Scale Visual Recognition Challenge}.
\newblock \emph{International Journal of Computer Vision (IJCV)}, 115\penalty0
  (3):\penalty0 211--252, 2015.
\newblock \doi{10.1007/s11263-015-0816-y}.

\bibitem[Sap et~al.(2019)Sap, Rashkin, Chen, Le~Bras, and
  Choi]{sap-etal-2019-social}
Maarten Sap, Hannah Rashkin, Derek Chen, Ronan Le~Bras, and Yejin Choi.
\newblock Social {IQ}a: Commonsense reasoning about social interactions.
\newblock In \emph{Proceedings of the 2019 Conference on Empirical Methods in
  Natural Language Processing and the 9th International Joint Conference on
  Natural Language Processing (EMNLP-IJCNLP)}, pages 4463--4473, Hong Kong,
  China, November 2019. Association for Computational Linguistics.
\newblock \doi{10.18653/v1/D19-1454}.
\newblock URL \url{https://aclanthology.org/D19-1454}.

\bibitem[Schuster et~al.(2021{\natexlab{a}})Schuster, Song, Tromer, and
  Shmatikov]{schuster2021autocomplete}
Roei Schuster, Congzheng Song, Eran Tromer, and Vitaly Shmatikov.
\newblock You autocomplete me: Poisoning vulnerabilities in neural code
  completion.
\newblock In \emph{30th {USENIX} Security Symposium ({USENIX} Security 21)}.
  {USENIX} Association, August 2021{\natexlab{a}}.
\newblock URL
  \url{https://www.usenix.org/conference/usenixsecurity21/presentation/schuster}.

\bibitem[Schuster et~al.(2021{\natexlab{b}})Schuster, Fisch, and
  Barzilay]{schuster-etal-2021-get}
Tal Schuster, Adam Fisch, and Regina Barzilay.
\newblock Get your vitamin {C}! robust fact verification with contrastive
  evidence.
\newblock In \emph{Proceedings of the 2021 Conference of the North American
  Chapter of the Association for Computational Linguistics: Human Language
  Technologies}, pages 624--643, Online, June 2021{\natexlab{b}}. Association
  for Computational Linguistics.
\newblock \doi{10.18653/v1/2021.naacl-main.52}.
\newblock URL \url{https://aclanthology.org/2021.naacl-main.52}.

\bibitem[Silver et~al.(2017)Silver, Schrittwieser, Simonyan, Antonoglou, Huang,
  Guez, Hubert, Baker, Lai, Bolton, Chen, Lillicrap, Hui, Sifre, van~den
  Driessche, Graepel, and Hassabis]{alphago}
David Silver, Julian Schrittwieser, Karen Simonyan, Ioannis Antonoglou, Aja
  Huang, Arthur Guez, Thomas Hubert, Lucas Baker, Matthew Lai, Adrian Bolton,
  Yutian Chen, Timothy Lillicrap, Fan Hui, Laurent Sifre, George van~den
  Driessche, Thore Graepel, and Demis Hassabis.
\newblock Mastering the game of go without human knowledge.
\newblock \emph{Nature}, 550\penalty0 (7676):\penalty0 354--359, 2017.
\newblock \doi{10.1038/nature24270}.
\newblock URL \url{https://doi.org/10.1038/nature24270}.

\bibitem[Silver et~al.(2018)Silver, Hubert, Schrittwieser, Antonoglou, Lai,
  Guez, Lanctot, Sifre, Kumaran, Graepel, et~al.]{alphaZero}
David Silver, Thomas Hubert, Julian Schrittwieser, Ioannis Antonoglou, Matthew
  Lai, Arthur Guez, Marc Lanctot, Laurent Sifre, Dharshan Kumaran, Thore
  Graepel, et~al.
\newblock A general reinforcement learning algorithm that masters chess, shogi,
  and go through self-play.
\newblock \emph{Science}, 362\penalty0 (6419):\penalty0 1140--1144, 2018.

\bibitem[Svyatkovskiy et~al.(2020)Svyatkovskiy, Deng, Fu, and
  Sundaresan]{Svyatkovskiy_2020}
Alexey Svyatkovskiy, Shao~Kun Deng, Shengyu Fu, and Neel Sundaresan.
\newblock {IntelliCode} {Compose}: code generation using {Transformers}.
\newblock \emph{Proceedings of the 28th ACM Joint Meeting on European Software
  Engineering Conference and Symposium on the Foundations of Software
  Engineering}, Nov 2020.
\newblock \doi{10.1145/3368089.3417058}.
\newblock URL \url{http://dx.doi.org/10.1145/3368089.3417058}.

\bibitem[Talmor et~al.(2019)Talmor, Herzig, Lourie, and
  Berant]{talmor-etal-2019-commonsenseqa}
Alon Talmor, Jonathan Herzig, Nicholas Lourie, and Jonathan Berant.
\newblock {C}ommonsense{QA}: A question answering challenge targeting
  commonsense knowledge.
\newblock In \emph{Proceedings of the 2019 Conference of the North {A}merican
  Chapter of the Association for Computational Linguistics: Human Language
  Technologies, Volume 1 (Long and Short Papers)}, pages 4149--4158,
  Minneapolis, Minnesota, June 2019. Association for Computational Linguistics.
\newblock \doi{10.18653/v1/N19-1421}.
\newblock URL \url{https://aclanthology.org/N19-1421}.

\bibitem[Tesauro(1995)]{tesauro1995temporal}
Gerald Tesauro.
\newblock Temporal difference learning and {TD-Gammon}.
\newblock \emph{Communications of the ACM}, 38\penalty0 (3):\penalty0 58--68,
  1995.

\bibitem[Vaswani et~al.(2017)Vaswani, Shazeer, Parmar, Uszkoreit, Jones, Gomez,
  Kaiser, and Polosukhin]{transformer2017}
Ashish Vaswani, Noam Shazeer, Niki Parmar, Jakob Uszkoreit, Llion Jones,
  Aidan~N Gomez, \L~ukasz Kaiser, and Illia Polosukhin.
\newblock Attention is all you need.
\newblock In I.~Guyon, U.~V. Luxburg, S.~Bengio, H.~Wallach, R.~Fergus,
  S.~Vishwanathan, and R.~Garnett, editors, \emph{Advances in Neural
  Information Processing Systems}, volume~30. Curran Associates, Inc., 2017.
\newblock URL
  \url{https://proceedings.neurips.cc/paper/2017/file/3f5ee243547dee91fbd053c1c4a845aa-Paper.pdf}.

\bibitem[Wang et~al.(2018)Wang, Singh, Michael, Hill, Levy, and
  Bowman]{wang-etal-2018-glue}
Alex Wang, Amanpreet Singh, Julian Michael, Felix Hill, Omer Levy, and Samuel
  Bowman.
\newblock {GLUE}: A multi-task benchmark and analysis platform for natural
  language understanding.
\newblock In \emph{Proceedings of the 2018 {EMNLP} Workshop {B}lackbox{NLP}:
  Analyzing and Interpreting Neural Networks for {NLP}}, pages 353--355,
  Brussels, Belgium, November 2018. Association for Computational Linguistics.
\newblock \doi{10.18653/v1/W18-5446}.
\newblock URL \url{https://aclanthology.org/W18-5446}.

\bibitem[Wolf et~al.(2020)Wolf, Debut, Sanh, Chaumond, Delangue, Moi, Cistac,
  Rault, Louf, Funtowicz, Davison, Shleifer, von Platen, Ma, Jernite, Plu, Xu,
  Scao, Gugger, Drame, Lhoest, and Rush]{wolf-etal-2020-transformers}
Thomas Wolf, Lysandre Debut, Victor Sanh, Julien Chaumond, Clement Delangue,
  Anthony Moi, Pierric Cistac, Tim Rault, Rémi Louf, Morgan Funtowicz, Joe
  Davison, Sam Shleifer, Patrick von Platen, Clara Ma, Yacine Jernite, Julien
  Plu, Canwen Xu, Teven~Le Scao, Sylvain Gugger, Mariama Drame, Quentin Lhoest,
  and Alexander~M. Rush.
\newblock Transformers: State-of-the-art natural language processing.
\newblock In \emph{Proceedings of the 2020 Conference on Empirical Methods in
  Natural Language Processing: System Demonstrations}, pages 38--45, Online,
  October 2020. Association for Computational Linguistics.
\newblock URL \url{https://www.aclweb.org/anthology/2020.emnlp-demos.6}.

\bibitem[Zavershynskyi et~al.(2018)Zavershynskyi, Skidanov, and
  Polosukhin]{naps2018}
Maksym Zavershynskyi, Alexander Skidanov, and Illia Polosukhin.
\newblock {NAPS:} natural program synthesis dataset.
\newblock \emph{CoRR}, abs/1807.03168, 2018.
\newblock URL \url{http://arxiv.org/abs/1807.03168}.

\bibitem[Zellers et~al.(2018)Zellers, Bisk, Schwartz, and
  Choi]{zellers-etal-2018-swag}
Rowan Zellers, Yonatan Bisk, Roy Schwartz, and Yejin Choi.
\newblock {SWAG}: A large-scale adversarial dataset for grounded commonsense
  inference.
\newblock In \emph{Proceedings of the 2018 Conference on Empirical Methods in
  Natural Language Processing}, pages 93--104, Brussels, Belgium,
  October-November 2018. Association for Computational Linguistics.
\newblock \doi{10.18653/v1/D18-1009}.
\newblock URL \url{https://aclanthology.org/D18-1009}.

\bibitem[Zellers et~al.(2019)Zellers, Bisk, Farhadi, and Choi]{zellers2019vcr}
Rowan Zellers, Yonatan Bisk, Ali Farhadi, and Yejin Choi.
\newblock From recognition to cognition: Visual commonsense reasoning.
\newblock In \emph{The IEEE Conference on Computer Vision and Pattern
  Recognition (CVPR)}, June 2019.

\bibitem[Ziegler(2021)]{copilot}
Albert Ziegler.
\newblock Research recitation: A first look at rote learning in github copilot
  suggestions., June 2021.
\newblock \url{https://docs.github.com/en/github/copilot/research-recitation},
  Last accessed on 2021-11-01.

\end{thebibliography}



\clearpage

\appendix
\counterwithin{figure}{section}
\counterwithin{table}{section}

We provide complementary details, analysis and results in the following sections:\\
\\
\ref{sec:example_solutions} \quad Example solutions by enumerative models \hfill Page \pageref{sec:example_solutions}\\
\ref{sec:td_details} \quad Enumerative solvers details \hfill Page \pageref{sec:td_details}\\
\ref{sec:gpt3_details} \quad Language Model solvers details \hfill Page \pageref{sec:gpt3_details}\\
\ref{sec:NP} \quad NP-completeness \hfill Page \pageref{sec:NP}\\
\ref{ap:open} \quad Open problems \hfill Page \pageref{ap:open}\\
\ref{ap:compare} \quad Comparing puzzles to competitive-programming problems \hfill Page \pageref{ap:compare}\\
\ref{sec:study_details} \quad User Study Details \hfill Page \pageref{sec:study_details}\\
\ref{sec:hanoi} \quad Solution to Tower of Hanoi \hfill Page \pageref{sec:hanoi}

\section{Example solutions by enumerative models}\label{sec:example_solutions}
We provide examples of our examined enumerative solvers to three P3 puzzle in Figure~\ref{fig:example_solutions} on page \pageref{fig:example_solutions} (examples of LM solutions are found in the P3 repository). The solution to the first puzzle is general and will work for any other instance of this problem. For the two other puzzles, the obtained solutions are instance-specific and don't even use the input variables. Yet, it is possible that the logical steps to achieve the answer are implicitly executed by the model. To test this, we evaluate the solvers on other problem instances (i.e., puzzles originated from the same problem).

The solvers' solutions to the first puzzle in Figure~\ref{fig:example_solutions} are simpler than the one created by humans (though less efficient in terms of input length). This illustrates another potential use case of AI solvers: debugging puzzles by finding easy solutions.

\begin{figure}
\small
\begin{pyblock}
# Sum of digits.
def sat1(x: str, s: int=679):
    return s == sum([int(d) for d in x])
    
# B. Random forest solution.
def sol(s):  
    return ((chr(49))*(COPY(s)))
    
# B. Transformer solution.
def sol(s):  
    return ((COPY(s))*(str(1)))

# Human-written solution.
def sol(s):  
    return int(s/9) * '9' + str(s%9)
    
----
# Line intersection.
def sat2(e: List[int], a: int=2, b: int=-1, c: int=1, d: int=2021):
    x = e[0] / e[1]
    return abs(a * x + b - c * x - d) < 10 ** -5
    
# B. Random forest and B. Transformer solution (identical).
def sol(a, b, c, d):
    return ([2022, 1, ])
    
# Human-written solution.
def sol(a, b, c, d):
    return [d - b, a - c]
    
---
# Find the three slice indices that give the specific target in string s.
def sat3(inds: List[int], s: str="hello world", target: str="do"):
    i, j, k = inds
    return s[i:j:k] == target
    
# B. Random forest solution.
def sol(s, target):
    return ([12, 5, -(3), ])
    
# B. Transformer solution.
def sol(s, target):
    return ([11, 1, -(6), ])
    
# Human-written solution.
def sol(s, target):
    from itertools import product
    for i, j, k in product(range(-len(s) - 1, len(s) + 1), repeat=3):
        try:
            if s[i:j:k] == target:
                return [i, j, k]
        except (IndexError, ValueError):
            pass
\end{pyblock}

\caption{Example of three P3 puzzles and the solutions found by our examined solvers. The natural language description of each problem is provided for ease of read, but is hidden to these models. Human-written solutions are provided here for reference, but are also hidden from AI solvers.}
\label{fig:example_solutions}
\end{figure}
\section{Enumerative solvers details}\label{sec:td_details}
We train our random forest solver with the Python Skickit-learn library~\citep{scikit-learn}. The 
Transformer model is implemented on top of the Hugging Face repository~\citep{wolf-etal-2020-transformers}. We use GPUs for training the Transformer and for querying it for rule probabilities. All other computations are performed with CPUs. Making up to $10^4$ solution tries takes only a few seconds to a few tens of seconds, depending on the puzzle and the attempted solutions. Running up to $10^6$ solution tries usually takes less than an hour but for some puzzles can take longer. We run the solver in parallel on multiple puzzles to reduce the global computation time.

\paragraph{Solution validation.} Given the AST, a solution is generated in the form of a Python program (possibly multiple lines) that is evaluated by the Python interpreter to get an answer that is tested by the puzzle. To address long-running programs and infinite loops, timeout checks are added to the puzzles and to the solution during conversion from AST to Python. Alternatively, the programs could be evaluated in a sandbox as is done in programming competitions and as we did for the LM generators, though a sandbox imposes an additional overhead. 

\subsection{Vocabulary}\label{sec:grammar_details}
We use a grammar for a subset of Python covering the following basic objects: Booleans, unlimited-precision integers, floats, strings, lists, sets, dictionaries, generators, and tuples. 
Table \ref{tab:grammar} summarizes the grammar. These rules occur multiply, for instance the addition rule has instantiations for adding two strings, two integers, an integer and a float, etc., where each Python type corresponds to a non-terminal in our grammar. However, because Python is a duck-typed language, in several cases a variable can be used with multiple different types. To handle such programs, we also have a generic non-terminal which can correspond to any Python object, and this makes our grammar ambiguous. For instance, the program \pyline{1+1} can be parsed either as the sum of two integers or as the sum of two Python objects, also using a rule mapping an object to an integer. This latter program is a larger AST and hence will typically have lower probability, hence we have the advantages of types when possible but the flexibility to generate fully duck-typed code. In this manner we are able to parse puzzles from 138 of our 200 problems.  We also use this grammar to generate timed and safe Python code. In particular, we inject timing checks into comprehensions and loops, and we also add timing checks to potentially time-consuming operations such as exponentiation or string multiplication. This grammar is available upon request for researchers who wish to use it in further projects.

\begin{table}[ht]
    \centering
        \caption{The grammar for a subset of Python.}
        \resizebox{1\columnwidth}{!}{%
    \begin{tabular}{r|lr|lr|l}
    \toprule
    Rule name & rule & Rule name & rule & Rule name & rule \\\midrule
\pyline{!=}     & \pyline{(_)!=(_)}           & \pyline{[list]} & \pyline{[_]}                & \pyline{is not} & \pyline{(_)is not(_)}       \\
\pyline{&}      & \pyline{(_)&(_)}            & \pyline{\%}     & \pyline{(_)\%(_)}           & \pyline{issubset} & \pyline{(_).issubset(_)}    \\
\pyline{(tuple)} & \pyline{(_, _)}             & \pyline{\{set\}} & \pyline{\{_\}}              & \pyline{issuperset} & \pyline{(_).issuperset(_)}  \\
\pyline{(tuple)} & \pyline{(_, _, _)}          & \pyline{^}      & \pyline{(_)^(_)}            & \pyline{join}   & \pyline{(_).join(_)}        \\
\pyline{(tuple)} & \pyline{(_, _, _, _)}       & \pyline{abs}    & \pyline{abs(_)}             & \pyline{len}    & \pyline{len(_)}             \\
\pyline{*}      & \pyline{(_)*(_)}            & \pyline{all}    & \pyline{all(_)}             & \pyline{list}   & \pyline{list(_)}            \\
\pyline{**}     & \pyline{(_)**(_)}           & \pyline{and}    & \pyline{(_)and(_)}          & \pyline{log}    & \pyline{log(_)}             \\
\pyline{*=}     & \pyline{(_)*=(_)}           & \pyline{any}    & \pyline{any(_)}             & \pyline{max}    & \pyline{max(_)}             \\
\pyline{*=}     & \pyline{_ *= (_)}           & \pyline{append} & \pyline{(_).append(_)}      & \pyline{min}    & \pyline{min(_)}             \\
\pyline{*args}  & \pyline{*_}                 & \pyline{arg}    & \pyline{_, _}               & \pyline{not}    & \pyline{not (_)}            \\
\pyline{*args}  & \pyline{*_, **_}            & \pyline{arg}    & \pyline{_: _, _}            & \pyline{not in} & \pyline{(_) not in (_)}     \\
\pyline{+}      & \pyline{(_)+(_)}            & \pyline{assert} & \pyline{assert _}           & \pyline{or}     & \pyline{(_)or(_)}           \\
\pyline{+=}     & \pyline{(_)+=(_)}           & \pyline{assert} & \pyline{assert _, _}        & \pyline{ord}    & \pyline{ord(_)}             \\
\pyline{+=}     & \pyline{_ += (_)}           & \pyline{bool}   & \pyline{bool(_)}            & \pyline{range}  & \pyline{range(_)}           \\
\pyline{+unary} & \pyline{+(_)}               & \pyline{chr}    & \pyline{chr(_)}             & \pyline{range}  & \pyline{range(_, _)}        \\
\pyline{-}      & \pyline{(_)-(_)}            & \pyline{cos}    & \pyline{cos(_)}             & \pyline{range}  & \pyline{range(_, _, _)}     \\
\pyline{-=}     & \pyline{(_)-=(_)}           & \pyline{count}  & \pyline{(_).count(_)}       & \pyline{replace} & \pyline{(_).replace(_, _)}  \\
\pyline{-unary} & \pyline{-(_)}               & \pyline{def}    & \pyline{def _(_):     _}    & \pyline{return} & \pyline{return (_)}         \\
\pyline{/}      & \pyline{(_)/(_)}            & \pyline{def_ANY_tuple} & \pyline{(_)}                & \pyline{reversed} & \pyline{reversed(_)}        \\
\pyline{//}     & \pyline{(_)//(_)}           & \pyline{default_arg} & \pyline{_: _=_, _}          & \pyline{revsorted} & \pyline{sorted(_, reverse=True)} \\
\pyline{//=}    & \pyline{(_)//=(_)}          & \pyline{default_arg} & \pyline{_=_, _}             & \pyline{round}  & \pyline{round(_)}           \\
\pyline{:slice} & \pyline{_:_:_}              & \pyline{endswith} & \pyline{(_).endswith(_)}    & \pyline{round}  & \pyline{round(_, _)}        \\
\pyline{<}      & \pyline{(_)<(_)}            & \pyline{exp}    & \pyline{exp(_)}             & \pyline{set}    & \pyline{set(_)}             \\
\pyline{<<}     & \pyline{(_)<<(_)}           & \pyline{f_string} & \pyline{f'_'}               & \pyline{sin}    & \pyline{sin(_)}             \\
\pyline{<=}     & \pyline{(_)<=(_)}           & \pyline{float}  & \pyline{float(_)}           & \pyline{sorted} & \pyline{sorted(_)}          \\
\pyline{=}      & \pyline{(_)=(_)}            & \pyline{float-const} & \pyline{_._}                & \pyline{split}  & \pyline{(_).split(_)}       \\
\pyline{==}     & \pyline{(_)==(_)}           & \pyline{float-const-large} & \pyline{_._e_}              & \pyline{split}  & \pyline{(_).split()}        \\
\pyline{>}      & \pyline{(_)>(_)}            & \pyline{float-const-tiny} & \pyline{_._e-_}             & \pyline{startswith} & \pyline{(_).startswith(_)}  \\
\pyline{>=}     & \pyline{(_)>=(_)}           & \pyline{for}    & \pyline{for (_) in (_):     _} & \pyline{str}    & \pyline{str(_)}             \\
\pyline{COPY}   & \pyline{COPY(_)}            & \pyline{for}    & \pyline{for (_, _) in (_):     _} & \pyline{str-const} & \pyline{"_"}                \\
\pyline{[-1]}   & \pyline{(_)[-1]}            & \pyline{for_in_if} & \pyline{for _ in (_) if _}  & \pyline{sum}    & \pyline{sum(_)}             \\
\pyline{[-2]}   & \pyline{(_)[-2]}            & \pyline{formatted_value} & \pyline{\{_:_\}}            & \pyline{tuple}  & \pyline{tuple(_)}           \\
\pyline{[-3]}   & \pyline{(_)[-3]}            & \pyline{if}     & \pyline{if _:     _}        & \pyline{type}   & \pyline{type(_)}            \\
\pyline{[-4]}   & \pyline{(_)[-4]}            & \pyline{if}     & \pyline{if _:     _ else:     _} & \pyline{union}  & \pyline{(_).union(_)}       \\
\pyline{[0]}    & \pyline{(_)[0]}             & \pyline{ifExp}  & \pyline{(_) if (_) else (_)} & \pyline{zip}    & \pyline{zip(_, _)}          \\
\pyline{[1]}    & \pyline{(_)[1]}             & \pyline{in}     & \pyline{(_) in (_)}         & \pyline{zip}    & \pyline{zip(_, _, _)}       \\
\pyline{[2]}    & \pyline{(_)[2]}             & \pyline{index}  & \pyline{(_).index(_)}       & \pyline{|}      & \pyline{(_)|(_)}            \\
\pyline{[3]}    & \pyline{(_)[3]}             & \pyline{int}    & \pyline{int(_)}             &  & \\
\pyline{[i]}    & \pyline{(_)[_]}             & \pyline{is}     & \pyline{(_)is(_)}           &  & \\
\bottomrule
    \end{tabular} 
    }
    \label{tab:grammar}
\end{table}

\subsection{Transformer implementation}\label{sec:transformer_details}

\begin{figure}[ht]
    \centering
    \small
    \includegraphics[width=0.5\textwidth]{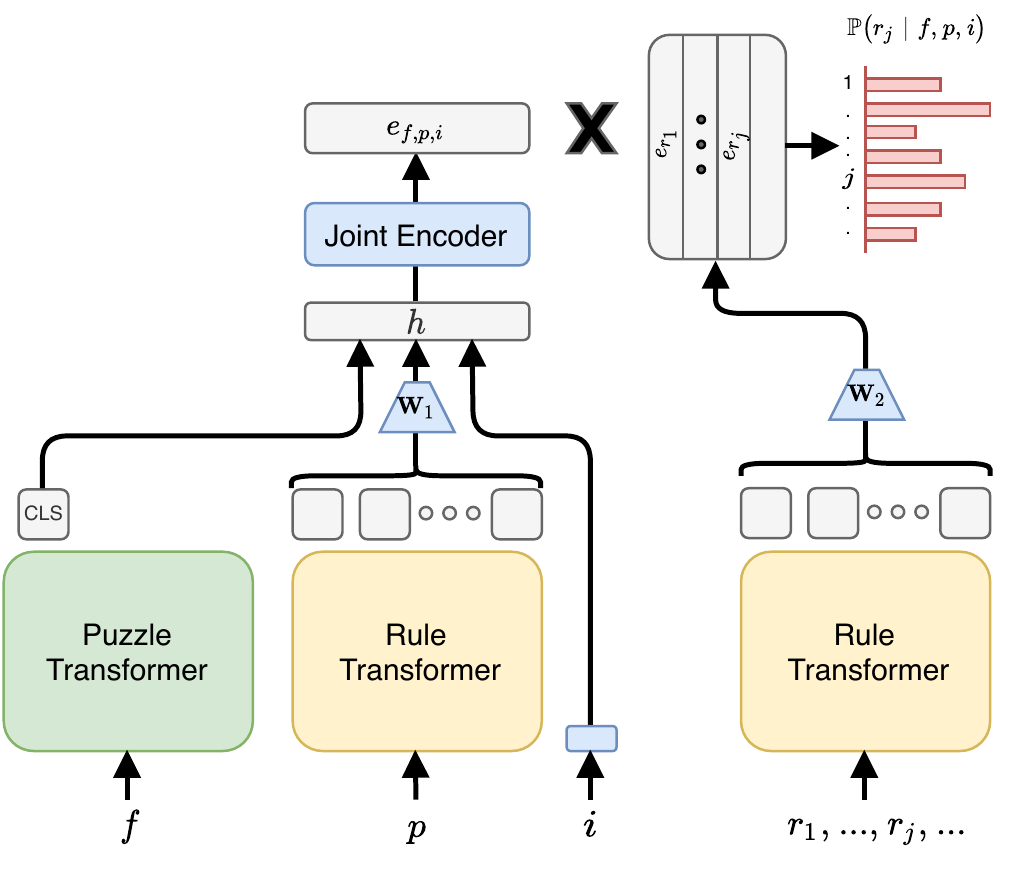}
    \caption{An illustration of our Transformer-based enumerative solver. The rule strings are encoded with a Transformer pretrained on Python code. The puzzle Transformer is initialized the same, but is further fine-tuned on puzzles, together with the rest of the solver's parameters shown in blue color. The left hand side of a diagram represents the encoding of the puzzle $f$, parent rule $p$, and child index $i$, each separately and then combined to a joint representation $e_{f,p,i}$. All rules $r\in\mathcal{R}$ are also encoded with the Transformer and projected to the same dimension as $e_{f,p,i}$. The output probability of $r$ being the $i$'s child of $p$ in the solution tree $g$ to puzzle $f$ is computed by a softmax over the product of $e_{f,p,i}$ with all rule representations. Encoding the puzzle and the parent rule first separately, allows passing the puzzle only once during inference, and computing all rule embeddings in advance.}
    \label{fig:transformer}
\end{figure}

We use the RoBERTa-base 12-layers Transformer~\citep{liu2019roberta} pretrained on English text and fine-tune it on Python code using the Hugging Face library~\citep{wolf-etal-2020-transformers}. For fine-tuning data, we use Python functions with their documentation text from GitHub repositories~\citep{husain2019codesearchnet}. In order to better adjust the tokenizer to Python code, we retrain a Byte-level BPE tokenizer on our Python fine-tuning data. We use the same vocabulary size as the original tokenizer and keep the token embeddings of the overlapping ones (39\%). For the other tokens, we initialize new token embeddings. Thereafter, we fine-tune RoBERTa with a masked language modeling task for 30 epochs. This model, which we denote by $T_P$, achieved an impressive $3.3$ perplexity score on held-out evaluation data, indicating its success in learning Python's syntax.

Next, we use $T_P$ to encode dense embeddings $e_r = T_P(r)$ for all the rules $r$ in our vocabulary $\mathcal{R}$. As input to the Transformer, we use a string representation of the Python operation and types of each rule. For example, \pyline{(x)//(y)' '// -> FLOAT :: (x: INT, y: FLOAT)} is used to describe the rule for the \pyline{//} operation with an integer and float inputs, resulting in a float. Then, we take $e_r$ as the average across the top-layer embeddings of all tokens. 

Finally, we design a neural model on top of $T_P$ to predict $\mathbb{P}(r_j|\phi(f), p, i)$ for each puzzle $f$ where $p$ is the parent rule and $i$ is the child index. The model computes a hidden representation of the puzzle with the parent rule as a concatenation $h=[T_P(f), \mathbf{W}_1 e_{p}, e_i] \in \mathbb{R}^{d+d_r+d_i}$, where $e_i \in \mathbb{R}^{d_i}$ is a learned embedding for the child rule index, $\mathbf{W}_1 \in \mathbb{R}^{d_r}$ is a learned linear projection, and $d$ is the hidden dimension of $T_P$. To obtain $\phi(f)$, we duplicate $T_P$ and further fine-tune it with the rest of the solver parameters, while keeping the rule Transformer fixed as $T_P$. Specifically, we use the $\texttt{[CLS]}$ embedding of the top layer as $\phi(f)$. Fixing the rule encoder prevents overfitting to the rules seen in the puzzle-solution fine-tuning pairs. $h$ is then passed through two non-linear layers, where the first also projects it to $\mathbb{R}^{d_r}$, with a gelu activation~\citep{hendrycks2016gaussian} and batch normalization~\citep{ioffe2015batch} to get a joint puzzle and parent rule embedding $e_{f,p,i}$. The score of rule $r_j$ then being the $i$'s argument of $r$ in the solution to $f$ is determined by the dot product of its projected embedding $e_{r_j}$ with the parent's embedding: $p_{r_j|\phi(f), e_{p}, e_i} \propto  e_{f,p,i} \cdot (\mathbf{W}_2 e_{r_j})^T$. Similar to the Random Forest fitting process, we use all parent-child rule pairs from the previously obtained solutions for fine-tuning. We use cross-entropy loss with an Adam optimizer. 
See Figure~\ref{fig:transformer} for a model diagram.
\section{Language Model solvers details}\label{sec:gpt3_details}
The GPT-3 and Codex APIs were used to generate completions based on prompts. For all models, the completions were generated in batches of \pyline{n=32} with \pyline{temp=0.9}, for a maximum of 150 tokens, with default values of  \pyline{top_p=1}, \pyline{presence_penalty=0}, \pyline{frequency_penalty=0}, and \pyline{best_of=1}. The resulting programs were evaluated in a sandbox limited to 1 second on Intel Xeon Platinum 8272CL CPUs at 2.60GHz. The timeout was necessary since a number of solution generators would take prohibitive resources such as \pyline{"a"*(10**(100))} which would generate a string of length googol. The solutions where also checked to be of the type requested in the problem, as was the case for the top-down solver. Figure \ref{fig:typecheck} illustrates a puzzle where type checking matters. 

\begin{figure}

\begin{pyblock}
def f(s: str):
    return s.count("o") == 1000 and s.count("oo") == 0
\end{pyblock}

\caption{A puzzle where type-checking is important. 
A type-safe solution is computed by the program returning \pyline{"ox" * 1000}. However, \pyline{["o"] * 1000} would be considered invalid as it is a list of strings, though it does satisfy the puzzle as stated.\label{fig:typecheck}}
\end{figure}

\paragraph{Prompt programming.}
The space of possible prompts is practically boundless. Our current prompt designs leverage the API without fine-tuning. For GPT-3, among the prompts we experimented with, we found that the \pyline{assert} structure worked best but it was limited to one-line Python solutions. One-line Python programs are considered, by some, to be a useful form of programming with books dedicated to the topic \citep[see, e.g.,][]{pythonOneLiners}. 
For Codex, we found a prompt that resembled a legal python file with a multi-line solution structure worked better.

Numerous other prompts were considered. For instance, we tried adding a preface stating, ``A Programming Puzzle is a short python function, and the goal is to find an input such that the function True. In other words, if program computes a function f, then the goal is to find x such that f(x)=True.'' 

\begin{figure}
\begin{pyblock}
def f1(s: str):
    return "Hello " + s == "Hello world"
    
assert True == f1("world")

---

def f2(s: str):
    return "Hello " + s[::-1] == "Hello world"

assert True == f2("world"[::-1])

---
def f3(x: List[int]):
    return len(x) == 2 and sum(x) == 3

assert True == f3([1, 2])

---
def f4(s: List[str]):
    return len(set(s)) == 1000 and all(
                   (x.count("a") > x.count("b")) and ('b' in x) for x in s)

assert True == f4(["a" * (i + 2) + "b" for i in range(1000)])

---

def f5(n: int):
    return str(n * n).startswith("123456789")

assert True == f5(int(int("123456789" + "0"*9) ** 0.5) + 1)

---

def f6(li: List[int]):
    return len(li) == 10 and li.count(li[3]) == 2

assert True == f6(...
\end{pyblock}

\caption{The medium-length prompt, used for GPT-3. The first five example puzzles \pyline{f1-f5} were shown to people in the user study and \pyline{f6} is the one that is being solved. GPT-3's completion was \pyline{...[1,2,3,3,4,5,6,7,8,9])}}
\label{fig:gpt_prompt_med}
\end{figure}

\begin{figure}
\begin{pysmall}
from typing import List

def f1(s: str):
    return "Hello " + s == "Hello world"

def g1():
    return "world"

assert f1(g1())

def f2(s: str):
    return "Hello " + s[::-1] == "Hello world"

def g2():
    return "world"[::-1]

assert f2(g2())

def f3(x: List[int]):
    return len(x) == 2 and sum(x) == 3

def g3():
    return [1, 2]

assert f3(g3())

def f4(s: List[str]):
    return len(set(s)) == 1000 and all((x.count("a") > x.count("b")) and ('b' in x) for x in s)

def g4():
    return ["a"*(i+2)+"b" for i in range(1000)]

assert f4(g4())

def f5(n: int):
    return str(n * n).startswith("123456789")

def g5():
    return int(int("123456789" + "0"*9) ** 0.5) + 1

assert f5(g5())

def f6(inds: List[int], string="Sssuubbstrissiingg"):
    return inds == sorted(inds) and "".join(string[i] for i in inds) == "substring"
    
def g6(string="Sssuubbstrissiingg"):
\end{pysmall}
Codex completed it successfully as:
\begin{pysmall}
    inds = []
    ind = 0
    for c in "substring":
        while string[ind] != c:
            ind += 1
        inds.append(ind)
        ind += 1
    return inds
\end{pysmall}

\caption{The medium-length prompt, used for Codex. The first five example puzzles \pyline{f1-f5} were given in the tutorial to participants in the user study and \pyline{f6} is the puzzle that is being solved.}
\label{fig:codex_prompt_med}
\end{figure}

\begin{figure}[ht]
\begin{pyblock}
def f1(s: str):
    """Find a string that when concatenated onto 'Hello ' gives 'Hello world'."""
    return "Hello " + s == "Hello world"
    
assert True == f1("world")

---

def f2(s: str):
    """Find a string that when reversed and concatenated onto 'Hello ' gives 'Hello world'."""
    return "Hello " + s[::-1] == "Hello world"
    
assert True == f2("world"[::-1])

---

def f3(x: List[int]):
    """Find a list of two integers whose sum is 3."""
    return len(x) == 2 and sum(x) == 3
    
assert True == f3([1, 2])

---

def f4(s: List[str]):
    """Find a list of 1000 distinct strings which each have more 'a's than 'b's and at least one 'b'."""
    return len(set(s)) == 1000 and all(
                   (x.count("a") > x.count("b")) and ('b' in x) for x in s)

assert True == f4(["a" * (i + 2) + "b" for i in range(1000)])

---

def f5(n: int):
    """Find an integer whose perfect square begins with 123456789 in its decimal representation."""
    return str(n * n).startswith("123456789")
    
assert True == f5(int(int("123456789" + "0"*9) ** 0.5) + 1)

---

def f6(li: List[int]):
    """Find a list of length 10 where the fourth element occurs exactly twice."""
    return len(li) == 10 and li.count(li[3]) == 2

assert True == f6(...
\end{pyblock}

\caption{An example GPT-3 Long prompt which includes English descriptions in the Python docstrings. As in the medium-length prompts, the first five example puzzles \pyline{f1-f5} were shown to people in the user study and \pyline{f6} is the one that is being solved. 
}
\label{fig:gpt_prompt_long}

\end{figure}

\begin{figure}[ht]
\begin{pysmall}
from typing import List

def f1(s: str):
    return "Hello " + s == "Hello world"

def g1():
    """Find a string that when concatenated onto 'Hello ' gives 'Hello world'."""
    return "world"

assert f1(g1())

def f2(s: str):
    return "Hello " + s[::-1] == "Hello world"

def g2():
    """Find a string that when reversed and concatenated onto 'Hello ' gives 'Hello world'."""
    return "world"[::-1]

assert f2(g2())

def f3(x: List[int]):
    return len(x) == 2 and sum(x) == 3

def g3():
    """Find a list of two integers whose sum is 3."""
    return [1, 2]

assert f3(g3())

def f4(s: List[str]):
    return len(set(s)) == 1000 and all(
                   (x.count("a") > x.count("b")) and ('b' in x) for x in s)

def g4():
    """Find a list of 1000 distinct strings which each have more 'a's than 'b's and at least one 'b'."""
    return ["a"*(i+2)+"b" for i in range(1000)]

assert f4(g4())

def f5(n: int):
    return str(n * n).startswith("123456789")

def g5():
    """Find an integer whose perfect square begins with 123456789 in its decimal representation."""
    return int(int("123456789" + "0"*9) ** 0.5) + 1

assert f5(g5())

def sat(d: int, n=6002685529):
        return n % d == 0 and all(i in "47" for i in str(d))
        
def sol(n=6002685529):
    """Find a integer factor of n whose decimal representation consists only of 7's and 4's."""
\end{pysmall}
\caption{An example Codex long prompt which includes English descriptions in the Python docstrings. As in the medium-length prompts, the first five example puzzles \pyline{f1-f5} were shown to people in the user study and \pyline{f6} is the one that is being solved. The representation here supports multi-line solutions. Codex's successful completion is shown in Figure \ref{fig:codex_prompt_long_completion}.
}
\label{fig:codex_prompt_long}

\end{figure}

\begin{figure}
\begin{pyblock}
    # return the closest factor to n which is inferior either to n or to the sqrt(n)
        if n == 1:     # def.: n>1!
            return 1
        i = int(n ** 0.5)
        while True:
            if f6(i):
                return i
            elif i == 2:
                raise Exception("Could not find any 7 or 4 number factors for n.")
            i -= 1
\end{pyblock}
\caption{An example Codex completion to the long prompt illustrated in \ref{fig:codex_prompt_long}. The solution cleverly calls \pyline{f6}, a trick that some participants used to solve a couple of the user study puzzles. Codex also generated English comments.
}
\label{fig:codex_prompt_long_completion}
\end{figure}

Interestingly, a handful of generations included potentially dangerous commands such as \pyline{eval} and \pyline{__import__("os").system(...}, but a cursory inspection did not detect any that used them in malicious ways. We do advise caution in executing generated code, as malicious actors can take advantage of such routine \citep{schuster2021autocomplete}. There are several libraries available for scoring programming competitions to serve this purpose. Also, some of the generated code seemed especially human-like, e.g.: \pyline{raise RuntimeError("this is a joke.")} which of course did not solve the puzzle at hand.

Figures~\ref{fig:gpt_prompt_short}, \ref{fig:gpt_prompt_med}, \ref{fig:codex_prompt_long}-\ref{fig:gpt_prompt_long}, and \ref{fig:prompt_bootstrap} show our prompts for the Short, Medium, Long, and Bootstrap prompts, respectively.

\begin{figure*}

\begin{pysmall}
from typing import List

def f1(item: int, li=[17, 2, 3, 9, 11, 11], index=4):
    return li.index(item) == index

def g1(li=[17, 2, 3, 9, 11, 11], index=4):
    return li[index]

assert f1(g1())

def f2(s: str, word="antidisestablishmentarianism", max_len=10):
    if len(word) <= max_len:
        return word == s
    return int(s[1:-1]) == len(word[1:-1]) and word[0] == s[0] and word[-1] == s[-1]

def g2(word="antidisestablishmentarianism", max_len=10):
    if len(word) == max_len:
        return word
    return word[0] + str(len(word[1:-1])) + word[-1]

assert f2(g2())

#
# omitting 28 random puzzles that Codex solved...
#

assert f31(g31())

def f32(x: List[int], a=7, s=5, e=200):
    return x[0] == a and x[-1] <= e and (x[-1] + s > e) and all([x[i] + s == x[i+1] for i in range(len(x)-1)])

def g32(a=7, s=5, e=200):
\end{pysmall}

\begin{pysmall}
def f1(s: str, a: List[str]=['cat', 'dot', 'bird'], b: List[str]=['tree', 'fly', 'dot']):
    return s in a and s in b

assert True == f1('dot')

---

def f2(li: List[int]):
    return all([sum(li[:i]) == i for i in range(20)])

assert True == f2(list(map(lambda x: 1, range(100))))

#
# omitting 22 random puzzles that GPT-3 solved...
#

---

def f25(probs: List[float]):
    assert len(probs) == 3 and abs(sum(probs) - 1) < 1e-6
    return max(probs[(i + 2) % 3] - probs[(i + 1) % 3] for i in range(3)) < 1e-6

assert True == f25(
\end{pysmall}

\caption{Example bootstrapping prompts for the Codex and GPT-3 LMs. The prompts includes random solved puzzles among those that the LM solved, truncated to the token limit of the API (2048 for GPT3 and 4096 for Codex).
}
\label{fig:prompt_bootstrap}

\end{figure*}

\paragraph{Smoothing evaluation.} Rather than simply generating solutions until the first correct one is found, to evaluate the Short, Medium and Long prompts, we generate 10,000 solutions for each puzzle. This gives us more than one solution for some puzzles, which we use for improved accuracy in estimating how many solutions are necessary (on average) to solve each puzzle shown in Figure \ref{fig:gpt3_cum}. We use the unbiased estimator of $\mathbf{pass}@k$ defined by \citet{chen2021evaluating}.

\section{NP-completeness}\label{sec:NP}
Before formally proving that the puzzle decision problem is NP-complete, note that the Boolean Satisfiability problem (SAT) is NP-complete and any Boolean SAT formula such as $(x_0 \vee \neg x_7 \vee x_{17}) \wedge \ldots$ can trivially be rewritten as a puzzle, e.g., 
\begin{pyblock}
def f(x: List[bool]):
   return (x[0] or not x[7] or x[17]) and ...
\end{pyblock}
The size of \pyline{f} is linear in the formula size. Thus converting a SAT formula to a puzzle is natural and does not make the problem much bigger or harder. 

However, a common misconception is that NP-complete problems are all equally intractable, but the theory of NP-completeness only speaks to the worst-case complexity of solving all puzzles. While any of our puzzles could theoretically be converted to a SAT formula, the resulting formula would be mammoth without any abstraction or intuition. For example, consider the following puzzle,
\begin{pyblock}
def f(d: int):  # find a non-trivial integer factor
    """Hint, try d = 618970019642690137449562111 ;-)"""
    n = 100433627766186892221372630609062766858404681029709092356097
    return 0 < d < n and n % d == 0
\end{pyblock}
This puzzle is identical to the factoring puzzle \pyline{f3} from Figure \ref{fig:examples} except that the answer is given away in a comment. Any natural compiler from Python to SAT would ignore comments so the SAT form of this trivial puzzle would be quite hard. While we are not aware of such a compiler, there are programs that convert a factoring problem to a SAT instance. We ran such a  converter \href{http://cgi.cs.indiana.edu/~sabry/cnf.cgi?factor=100433627766186892221372630609062766858404681029709092356097&Adder=nbit&Multiplier=carrysave}{http://cgi.cs.indiana.edu/~sabry/cnf.html} on this \pyline{n} and it generated a formula with 113,878 variables and 454,633 terms! This illustrates that not all polynomials are small, and that some easy puzzles may become hard puzzles in such a conversion. The theory of NP-completeness only guarantees that if one can efficiently solve \textit{every} SAT instance one could efficiently solve every puzzle, but specific easy puzzles may become quite hard SAT formulas. 

\subsection{Proof of NP-completeness}
Formally, a puzzle $f$ represents a Turing machine as a string, a timeout $t$ is a positive integer represented in unary, and the decision problem is, given $(f,x,t)$, does there exist $y$ such that when the Turing machine $f$ is run on $(y, x)$, it halts in fewer than $t$ steps and outputs 1. The time constraint is necessary to ensure that the puzzle decision problem is in $\NP$. It is well-known that this problem is in NP and, moreover is NP-complete:
\begin{observation}
The puzzle decision problem is NP-complete.
\end{observation}
\begin{proof}
One can test whether a given puzzle string $f$ encoding a Turing machine halts on a witness $y$ in time $\leq t$ by simulating running $f$ on $(y, x)$ for $t$ steps. Since simulating a Turing machine of size $|f|$ running for $t$ steps can be done in $\poly(|f|, t)$ time, this can be done in time $\poly(|f|, t)$ as required for $\NP$.

To see that the problem is complete, note that given any other NP problem defined by a Turing machine $T(x, y)$ that runs on input $x\in \Sigma^*$ and witness $y \in \Sigma^*$ in polynomial time $t = p(|x|)$ is a type of puzzle itself for $f=T$ (with inputs swapped).
\end{proof}
\section{Open problems}\label{ap:open}
The following five puzzles would each represent a major breakthrough in computer science or mathematics if solved. 
\begin{enumerate}
    \item \href{https://github.com/microsoft/PythonProgrammingPuzzles/blob/main/problems/README.md#factoring}{Factoring}. In the traditional version of this ancient problem, the goal is to efficiently find the prime factorization of a given integer. In the puzzle version, we state the equivalent problem of finding any non-trivial factor of a given integer. The puzzle is equivalent in the sense that one can recursively call the puzzle on each of the factors found until one achieves the complete prime factorization. A number of factoring algorithms have been developed over decades that factor larger and larger numbers. The RSA Factoring Challenge \citep[see, e.g.,][]{Kaliski2005} has awarded tens of thousands of dollars in prize money and RSA offered \$200,000 for factoring the largest RSA challenge number with 617 digits. The closely related \href{https://github.com/microsoft/PythonProgrammingPuzzles/blob/main/problems/README.md#discretelog}{Discrete Log} problem is also unsolved.
    
    \item \href{https://github.com/microsoft/PythonProgrammingPuzzles/blob/main/problems/README.md#GraphIsomorphism}{Graph Isomorphism}. Given two isomorphic graphs, find the bijection that relates the two of them. In a breakthrough, Babai has claimed a quasi-polynomial time for this problem, but no polynomial time algorithm is known.
    
    \item \href{https://github.com/microsoft/PythonProgrammingPuzzles/blob/main/problems/README.md#plantedclique}{Planted Clique}. In this classic graph-theory problem, an $n$-node Erdős–Rényi random graph random graph is chosen and then $k$ nodes are selected at random and the edges are added so that they form a clique. The problem is to find the clique. It is not known whether there is a polynomial-time algorithm for this problem \citep[see, e.g.,][]{Arora2009ComputationalCA}.

    \item \href{https://github.com/microsoft/PythonProgrammingPuzzles/blob/main/problems/README.md#learnparitywithnoise}{Learning Parity with Noise}. This is a binary classification problem in computational learning theory. Roughly speaking, the problem is to efficiently learn a parity function with random classification noise. The fastest known algorithm for this problem runs in time $\tilde{O}(2^{n/\log n})$ \citep{blum2003}. The problem is also closely related to efficiently decoding random linear codes \citep{berlekamp1978} and various assumptions in cryptography. Note that some of the instances of this problem are small (and thus easy) while others are quite large.

    \item \href{https://github.com/microsoft/PythonProgrammingPuzzles/blob/main/problems/README.md#collatzcycleunsolved}{Collatz cycle}. The problem is to find a cycle in the famous $3n+1$ process, where you start with integer $n>0$ and repeatedly set $n$ to $n/2$ if $n$ is even, otherwise $3n+1$, until you reach 1. The Collatz cycle conjecture is that there are no cycles in this process. According to the \href{https://en.wikipedia.org/wiki/Collatz_conjecture}{Wikipedia article} on the topic, Jeffrey Lagarias stated that it ``is an extraordinarily difficult problem, completely out of reach of present day mathematics'' and Paul Erdős said ``Mathematics may not be ready for such problems.'' He also offered \$500 for its solution. 
    
\end{enumerate}
Each of these problems is described by a short (1-5 line) python function. Now, for the algorithms problems 1-3, the puzzle involves solving given instances and not exactly with the open problem: coming up with a provably polynomial-time algorithm, and it is entirely possible that no poly-time algorithm exists. However, these are all problems that have been intensely studied and an improvement, even a practical one, would be a breakthrough. For the Collatz cycle, if the Collatz conjecture holds then there is no cycle. However, we give problems involving finding integers with large Collatz delays which could be used to, at least, break records. 
Also noteworthy but perhaps not as well-known is \href{https://github.com/microsoft/PythonProgrammingPuzzles/blob/main/problems/README.md#conway99}{Conway's 99 puzzle}, an unsolved problem in graph theory due to Conway and \cite{biggs1971finite} (as cited by Wikipedia). The two-line puzzle describes finding an undirected graph with 99 vertices, in which each two adjacent vertices have exactly one common neighbor, and in which each two non-adjacent vertices have exactly two common neighbors. \citet{Conway99} offered \$1,000 for its solution. 

There are also several unsolved puzzles in terms of beating records, e.g., finding oscillators or spaceships of certain periods in Conway's game of life and finding uncrossed knights tours on chess boards of various sizes.

\section{Comparing puzzles to competitive-programming problems}\label{ap:compare}

Figure \ref{fig:codeforecompare} 
illustrates an elementary \url{codeforces.com} problem. As is typical in programming competitions, the authors have concocted an entertaining story to motivate the problem. \citet{dagiene2008bebras} include ``should be funny''  and ``should have pictures'' among desirable criteria for competitive programming problems. Also, as is typical the first step is explaining how the input is formatted and how the output should be formatted. One difficulty in authoring such competitive-programming challenges is ensuring that the English description unambiguously matches with the hidden test cases. The \href{https://icpc.global/worldfinals/rules}{ICPC rules} state: ``A contestant may submit a claim of ambiguity or error in a problem statement by submitting a clarification request.  If the judges agree that an ambiguity or error exists, a clarification will be issued to all contestants.'' With puzzles, this is not necessary---a mistake in a puzzle either means that the puzzle is unsolvable or that the puzzle has an unexpected (often trivial) solution, neither of which cause major problems as it would still be a fair comparison of different solvers.

The puzzle form \href{https://github.com/microsoft/PythonProgrammingPuzzles/blob/main/problems/README.md#invertpermutation}{InvertPermutation}\footnote{In P3, we have slightly modified the problem so that it is only inspired by the codeforces problem and not a direct translation. The P3 problem is harder in that characters not in the permutation may also appear in the string unmodified.} has no story, no description of input/output format, and no examples. The input/output formatting is taken care of simply by the type hints. 

The intention is for puzzles to isolate the essence of the part of the problem that involves reasoning. Other datasets already address natural language understanding and input/output string formatting.

\begin{figure}[ht]

\begin{mdframed}
{\large \textbf{Codeforces problem 474 A. Keyboard}}\hrule
\smallskip

Our good friend Mole is trying to code a big message. He is typing on an unusual keyboard with characters arranged in following way:\\

qwertyuiop\\
asdfghjkl;\\
zxcvbnm,./\\

Unfortunately Mole is blind, so sometimes it is problem for him to put his hands accurately. He accidentally moved both his hands with one position to the left or to the right. That means that now he presses not a button he wants, but one neighboring button (left or right, as specified in input).

We have a sequence of characters he has typed and we want to find the original message.\\

\textbf{Input}\\
First line of the input contains one letter describing direction of shifting ('L' or 'R' respectively for left or right).\\

Second line contains a sequence of characters written by Mole. The size of this sequence will be no more than 100. Sequence contains only symbols that appear on Mole's keyboard. It doesn't contain spaces as there is no space on Mole's keyboard.\\

It is guaranteed that even though Mole hands are moved, he is still pressing buttons on keyboard and not hitting outside it.\\

\textbf{Output}\\
Print a line that contains the original message.\\

{\large \textbf{Examples}}\\\\
\textbf{input}\\
R\\
s;;upimrrfod;pbr\\

\textbf{output}\\
allyouneedislove\\
\end{mdframed}
\begin{pysmall}
def f(s: str, perm="qwertyuiopasdfghjkl;zxcvbnm,./", target="s;;upimrrfod;pbr"):
    return "".join(perm[perm.index(c) + 1] for c in s) == target
\end{pysmall}
    \caption{Example of an introductory competition problem \url{https://codeforces.com/problemset/problem/474/A} (top) and the respective puzzle version (bottom) that is only using code and is short to read. In this problem, there is a given permutation of characters $\pi$, and a given target string $t$, and one wants to find a source string $s$ such that when each character of $s$ has been permuted with $\pi$, the target is achieved. The puzzle has been simplified to always shift right.}
    \label{fig:codeforecompare}
\end{figure}

\section{User Study Details}\label{sec:study_details}
The user study began with a short tutorial about puzzles, which included the puzzles shown in Figure \ref{fig:gpt_prompt_med}. The 30 puzzles (see Figures \ref{fig:study1}-\ref{fig:study2}) were divided into three parts of 10 puzzles each: numbers 1-10, 11-20, and 20-30. Since each puzzle took at maximum of 6 minutes, no part took more than one hour. In the internal IRB approval (July 22, 2020), the key discussion points were that we would not collect age, gender or any other PII since it was not relevant to our study.

\subsection{Provided instructions}


Figures~\ref{fig:hack_1}-\ref{fig:hack_3} present the initial instructions that participants were given before starting the study. Figures~\ref{fig:hack_nb_1}-\ref{fig:hack_nb_2} show the interface that they used for retrieving puzzles and submitting solutions. We run implement a Python backend to store progress logs and to serve each puzzle in its turn, so participants won't accidentally be exposed to any of the puzzles in advance. We asked participants to follow the simple interface and not to attempt any sophisticated hacking techniques that will give them any personal advantage. We did not observe any such malicious behaviour and received positive feedback for the stability and clarity of the interface.

\subsection{Qualitative feedback.} 
Our Jupyter notebook interface also allowed users to submit qualitative feedback. As an example of this last point, participants mentioned that they were not familiar with functions such as \pyline{zip} or \pyline{all} but learned them in the course of the study.
Overall, Three themes emerged in the feedback: participants enjoyed solving the puzzles, they felt that 6 minutes was not enough time to solve the puzzles, and they felt they learned Python from doing the puzzles.

\subsection{Results summary}
A total of 21 participants completed the user study. Participants solved between 12-30 puzzles, with 6 participants solving more than 28 puzzles, and only a single participant solving all 30. As Figure~\ref{fig:hack_yoe} shows, the participants Python experience ranged between a few months to 8 years, with a median of 3 years. For post study analysis purposes, we denote participants with less than 3 years of experience as \emph{beginners} and the rest as \emph{experienced}. Figure~\ref{fig:user_study_solved} shows the number of participants that solved each puzzle, grouped by experience. 9 of the puzzles were solved by all beginners, whereas 17 puzzles were solved by all experienced. This positive correlation between the number of programming experience and number of puzzles solved, indicates the effectiveness of our puzzles as a proxy to evaluating programming proficiency.

We also notice that experienced programmers solve puzzles faster (149 seconds per puzzle on average, compared to 194 seconds for beginners). Figure~\ref{fig:user_study} shows the distribution of time spent by participants on each puzzle. We use the per puzzle average solving time as an indicator to its perceived difficulty. As discussed in the main paper (\S\ref{sec:study_res}), we see a strong correlation between the perceived difficulty of different puzzles for humans and for our examined AI solvers. 

\begin{figure}[t]
    \centering
    \begin{mdframed}
    \includegraphics[width=1\textwidth]{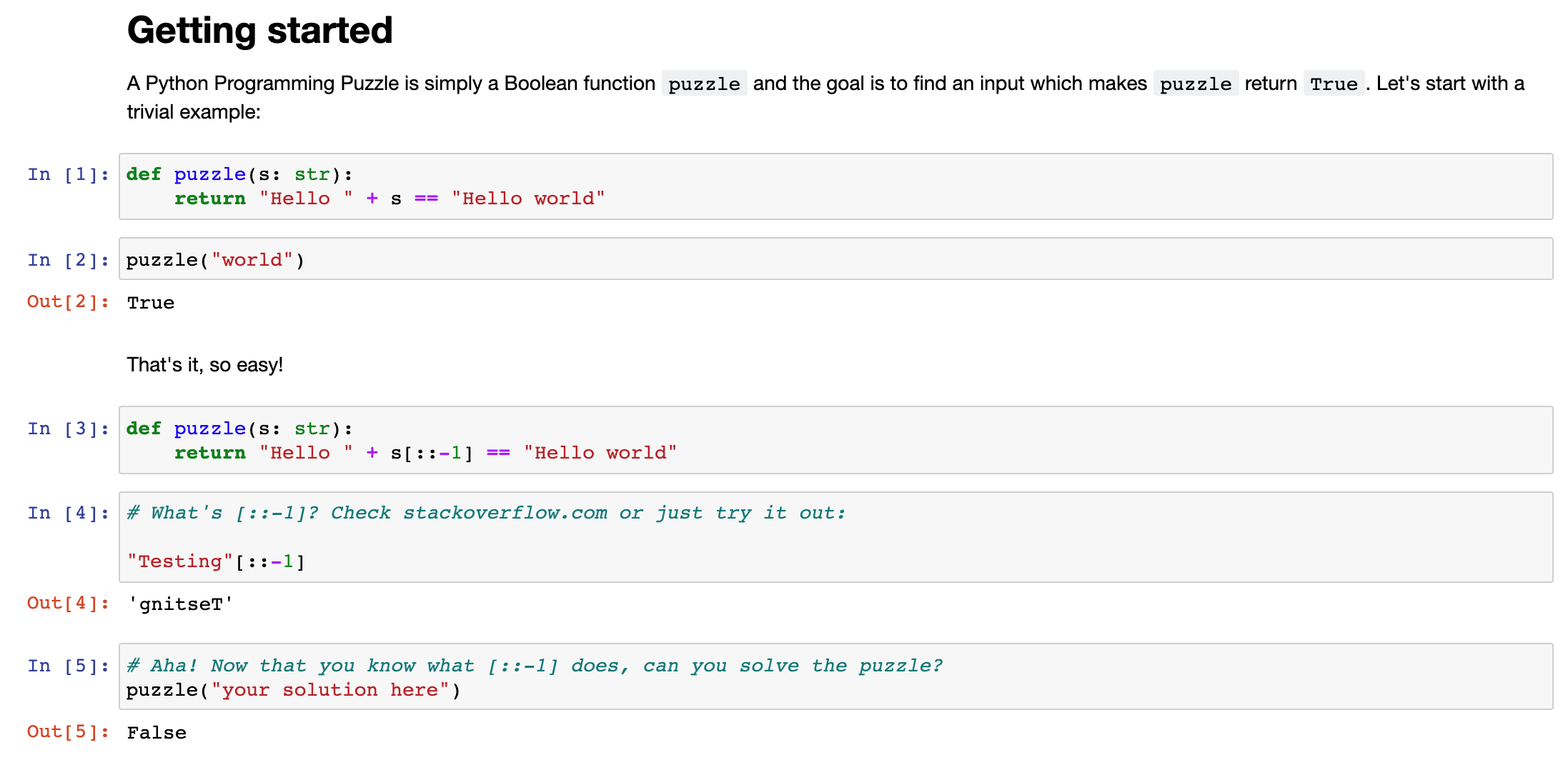}
    \end{mdframed}
    \caption{Instructions page provided to the study participants as a Jupyter notebook (part 1).}
    \label{fig:hack_1}
\end{figure}

\begin{figure}[t]
    \centering
    \begin{mdframed}
    \includegraphics[width=1\textwidth]{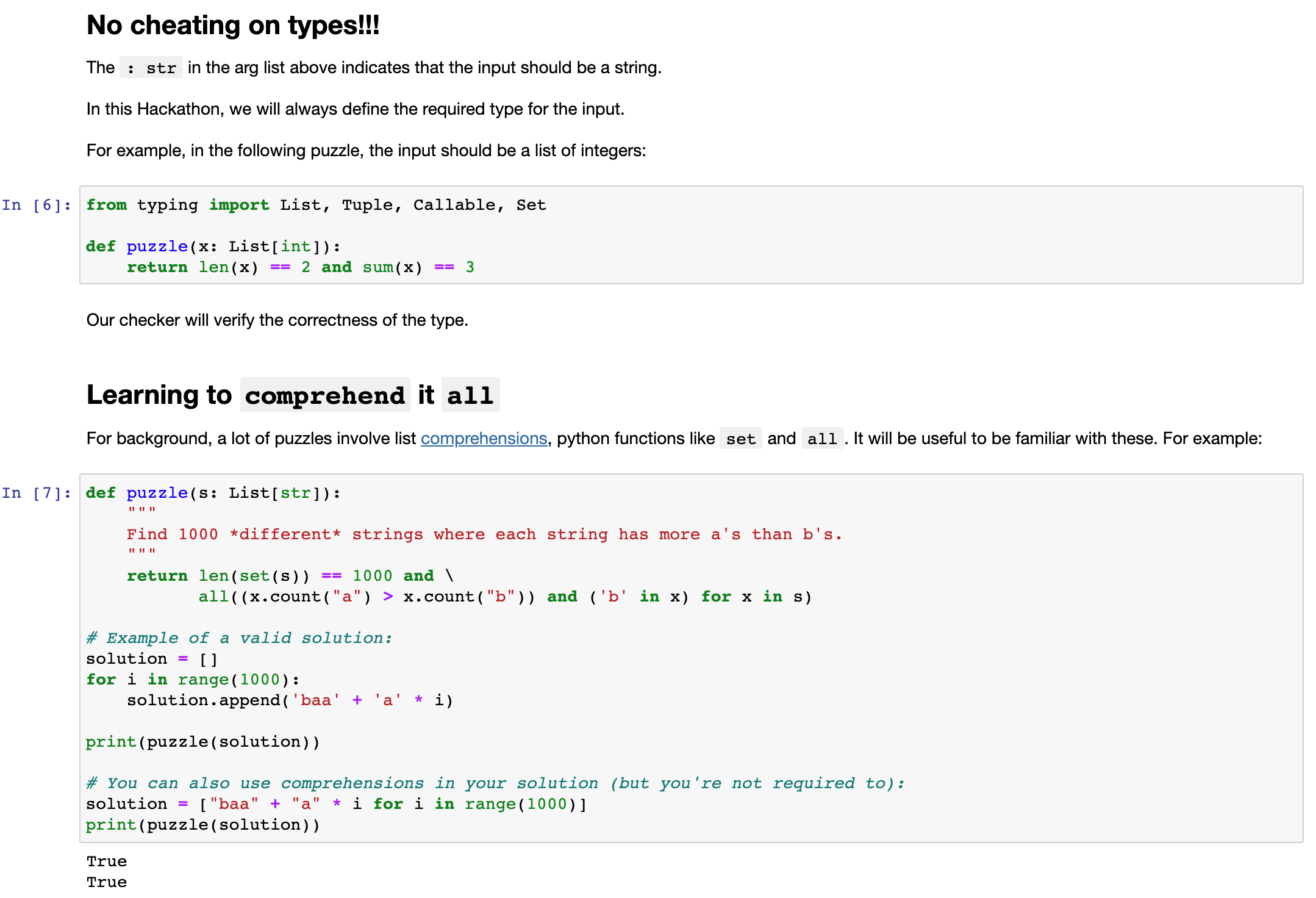}
    \end{mdframed}
    \caption{Instructions page provided to the study participants as a Jupyter notebook (part 2).}
    \label{fig:hack_2}
\end{figure}

\begin{figure}[t]
    \centering
    \begin{mdframed}
    \includegraphics[width=1\textwidth]{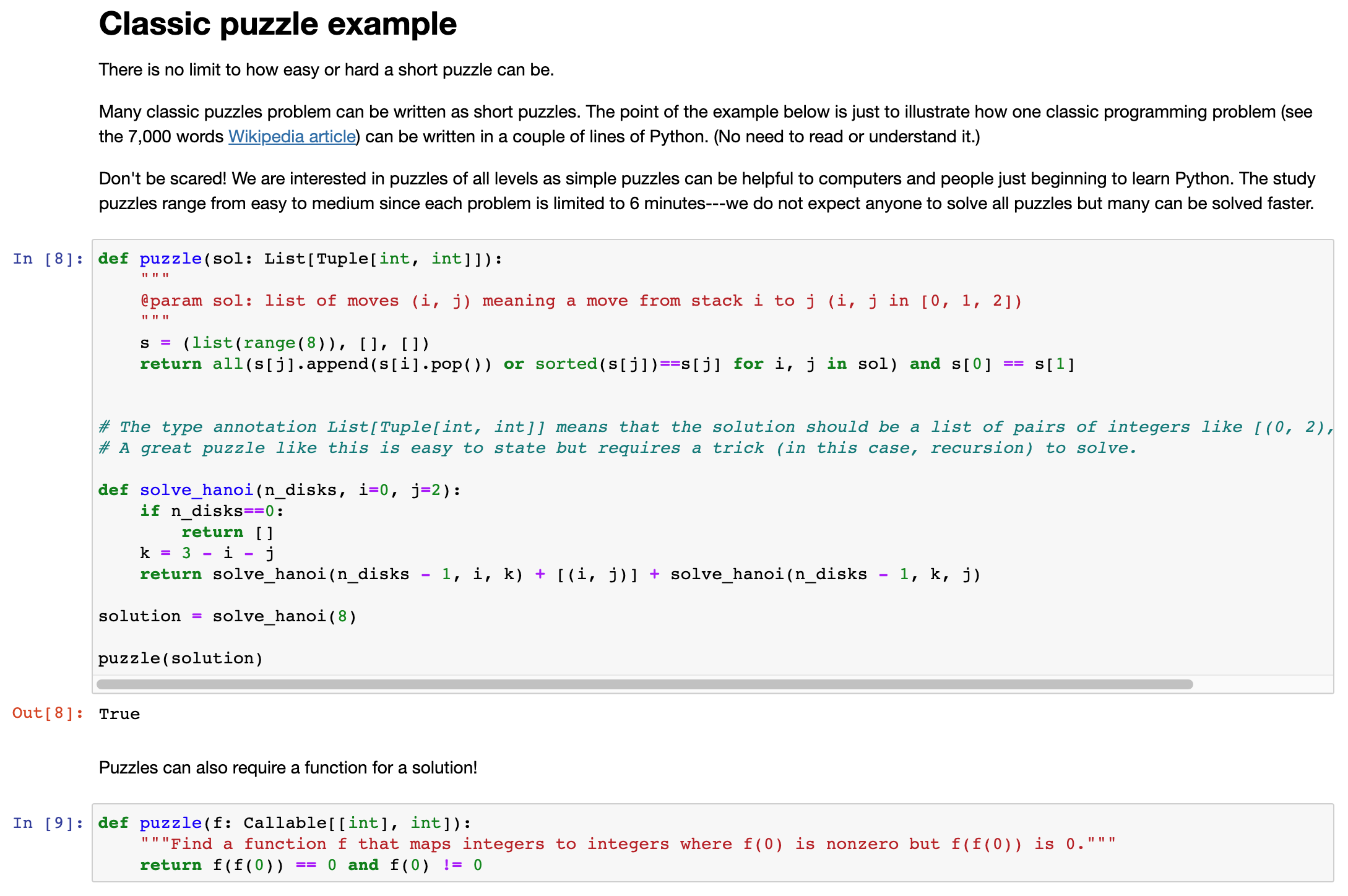}
    \end{mdframed}
    \caption{Instructions page provided to the study participants as a Jupyter notebook (part 3).}
    \label{fig:hack_3}
\end{figure}

\begin{figure}[t]
    \centering
    \begin{mdframed}
    \includegraphics[width=1\textwidth]{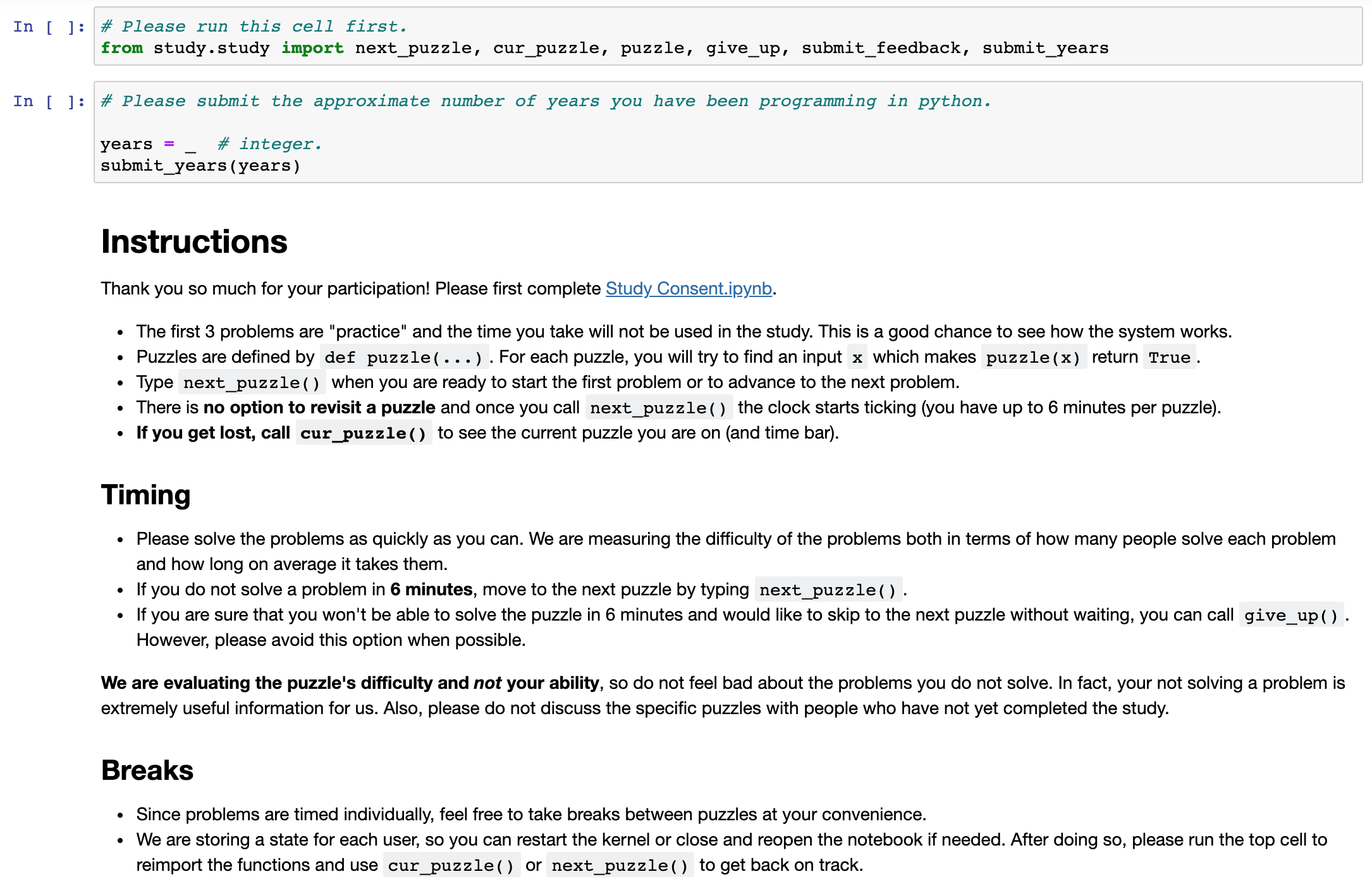}
    \end{mdframed}
    \caption{The introduction of the study notebook given to participants.}
    \label{fig:hack_nb_1}
\end{figure}

\begin{figure}[t]
\small
    \centering
    \begin{subfigure}[b]{1\textwidth}
    \begin{mdframed}
    \includegraphics[width=1\textwidth]{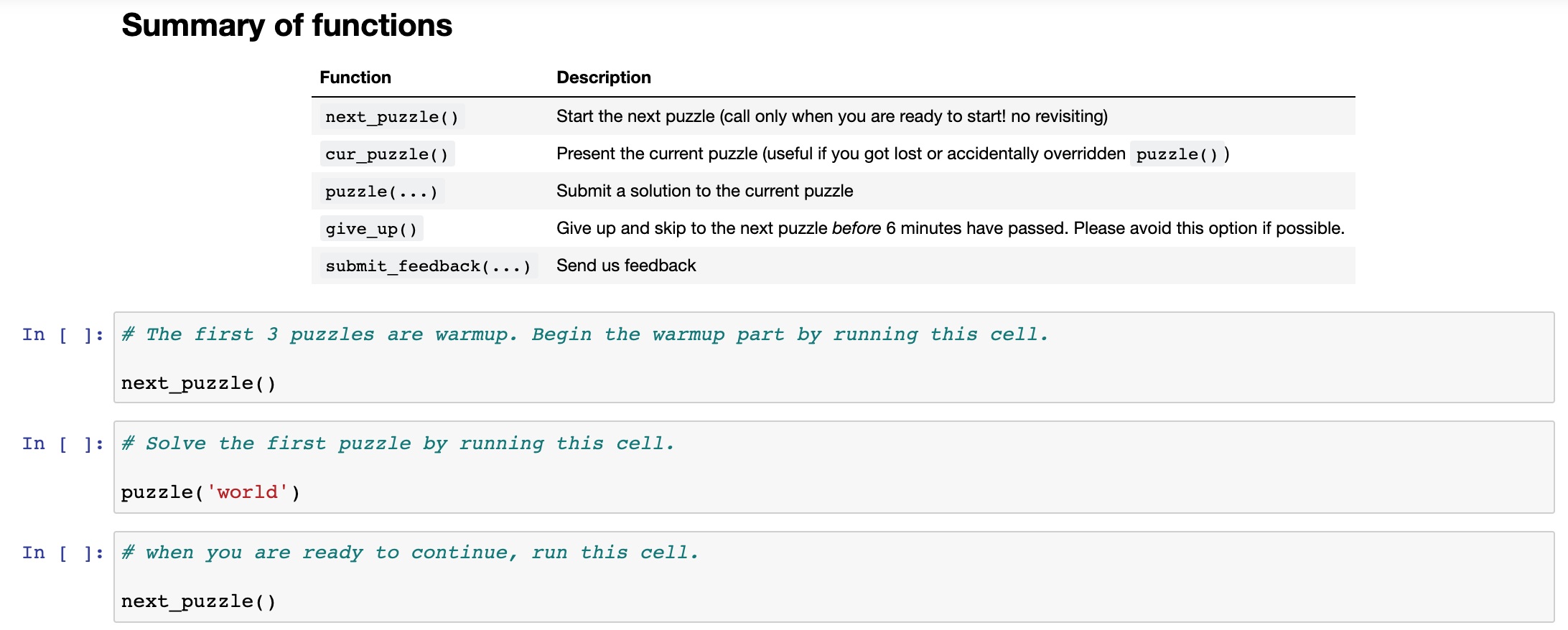}
    \end{mdframed}
    \caption{Initial view.}
    \end{subfigure}
    
    \begin{subfigure}[b]{1\textwidth}
    \includegraphics[width=1\textwidth]{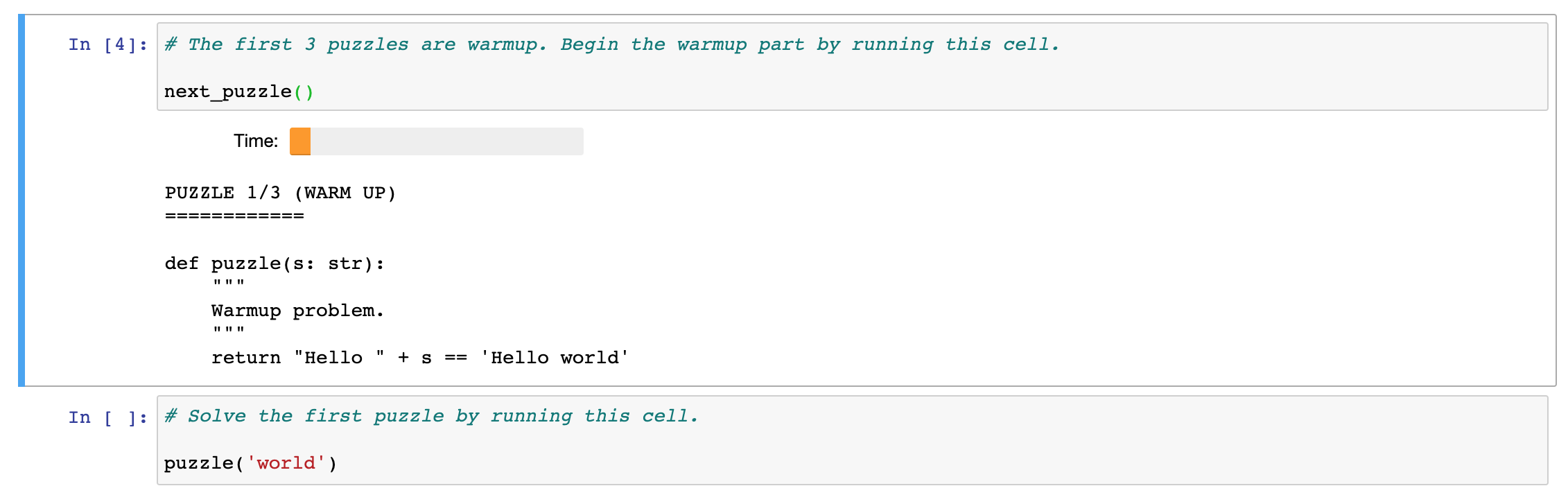}
    \caption{View while solving a puzzle. The progress bar advances towards the 6 minutes limit.}
    \end{subfigure}
    
        \begin{subfigure}[b]{1\textwidth}
    \includegraphics[width=1\textwidth]{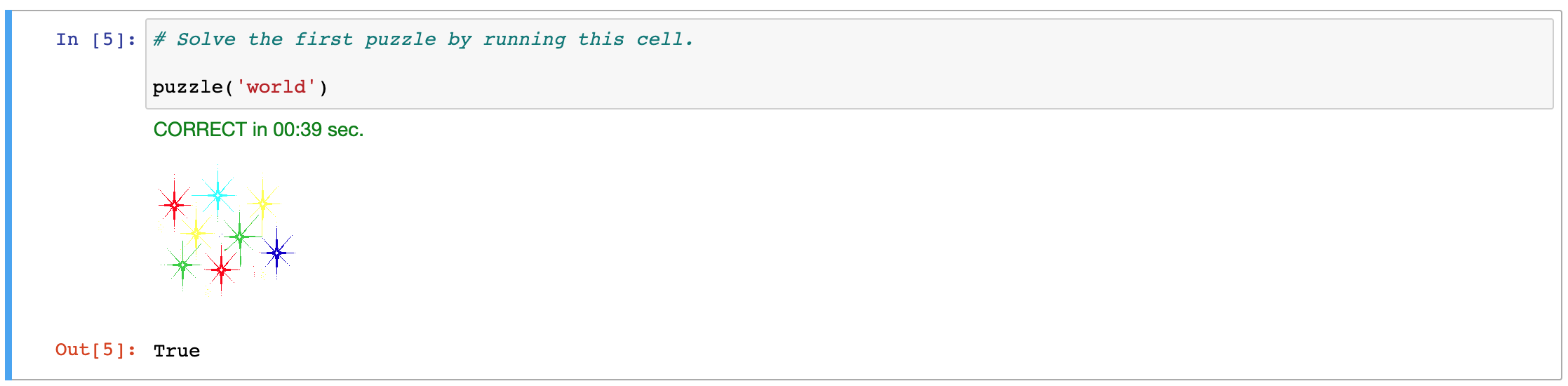}
    \caption{View after submitting a successful solution.}
    \end{subfigure}
    
            \begin{subfigure}[b]{1\textwidth}
            \begin{mdframed}
    \includegraphics[width=1\textwidth]{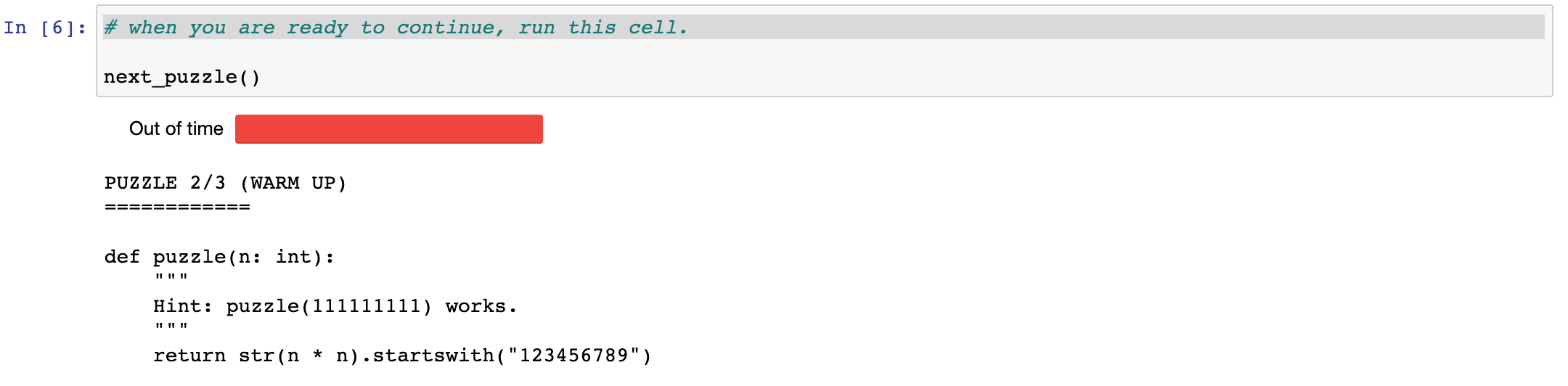}
    \end{mdframed}
    \caption{View after 6 minutes have passed since viewing the puzzle without submitting a valid solution.}
    \end{subfigure}
    
    \begin{subfigure}[b]{1\textwidth}
    \begin{mdframed}
    \includegraphics[width=1\textwidth]{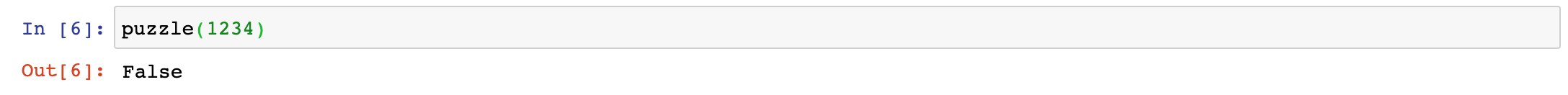}
    \end{mdframed}
    \caption{View when submitting a wrong solution to a puzzle (before timeout is reached).}
    \end{subfigure}
    
    \caption{The interface used by participants to solve puzzles during the study. Each sub-figure shows a different state of the notebook according to the user's interaction.}
    \label{fig:hack_nb_2}
\end{figure}

    

\begin{figure}
\begin{pysmall}
def f1(s: str):
    return s.count("o") == 1000 and s.count("oo") == 100 and s.count("ho") == 801

def f2(s: str):
    return s.count("o") == 1000 and s.count("oo") == 0

def f3(x: List[int]):
    return sorted(x) == list(range(999)) and all(x[i] != i for i in range(len(x)))

def f4(x: List[int]):
    return len(x) == 10 and x.count(x[3]) == 2

def f5(x: List[int]):
    return all([x.count(i) == i for i in range(10)]) 

def f6(n: int):
    return n % 123 == 4 and n > 10**10

def f7(s: str):
    return str(8**2888).count(s) > 8 and len(s) == 3

def f8(s: List[str]):
    return s[1234] in s[1235] and s[1234] != s[1235]

def f9(x: List[int]):
    return ["The quick brown fox jumps over the lazy dog"[i] for i in x] \
            == list("The five boxing wizards jump quickly")

def f10(s: str):
     return s in str(8**1818) and s==s[::-1] and len(s)>11
     
def f11(x: List[str]):
    return min(x) == max(x) == str(len(x))

def f12(x: List[int]):
    return all(a + b == 9 for a, b in zip([4] + x, x)) and len(x) == 1000

def f13(x: float):
    return str(x - 3.1415).startswith("123.456")

def f14(x: List[int]):
   return all([sum(x[:i]) == i for i in range(20)])

def f15(x: List[int]):
   return all(sum(x[:i]) == 2 ** i - 1 for i in range(20)) 
\end{pysmall}
    \caption{The first 15 puzzles in the user study.}
    \label{fig:study1}
\end{figure}
\begin{figure}
\begin{pysmall}
def f16(x: str):
    return float(x) + len(x) == 4.5

def f17(n: int):
    return len(str(n + 1000)) > len(str(n + 1001)) 

def f18(x: List[str]): 
    return [s + t for s in x for t in x if s!=t] == 'berlin berger linber linger gerber gerlin'.split()

def f19(x: Set[int]):
    return {i+j for i in x for j in x} == {0, 1, 2, 3, 4, 5, 6, 17, 18, 19, 20, 34}

def f20(x: List[int]):    
    return all(b in {a-1, a+1, 3*a} for a, b in zip([0] + x, x + [128]))

def f21(x: List[int]):
   return all([x[i] != x[i + 1] for i in range(10)]) and len(set(x)) == 3

def f22(x: str):
    return x[::2] in x and len(set(x)) == 5

def f23(x: List[str]):
    return tuple(x) in zip('dee', 'doo', 'dah!')

def f24(x: List[int]):
    return x.count(17) == 3 and x.count(3) >= 2

def f25(s: str):
    return sorted(s)==sorted('Permute me true') and s==s[::-1]

def f26(x: List[str]):
   return "".join(x) == str(8**88) and all(len(s)==8 for s in x)

def f27(x: List[int]):
   return x[x[0]] != x[x[1]] and x[x[x[0]]] == x[x[x[1]]]

def f28(x: Set[int]):
   return all(i in range(1000) and abs(i-j) >= 10 for i in x for j in x if i != j) \
          and len(x)==100  

def f29(x: Set[int]):
   return all(i in range(1000) and abs(i*i - j*j) >= 10 for i in x for j in x if i != j) and len(x) > 995

def f30(x: List[int]):
    return all([123*x[i] % 1000 < 123*x[i+1] % 1000 and x[i] in range(1000) 
                for i in range(20)])
\end{pysmall}
    \caption{The last 15 puzzles in the user study.}
    \label{fig:study2}
\end{figure}

\begin{figure}[ht]
    \centering
    \begin{subfigure}[b]{0.45\textwidth}
    \includegraphics[width=1\textwidth]{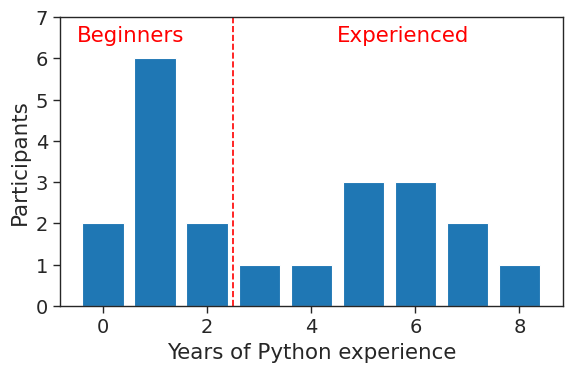}
    \end{subfigure}
    \hfill
    \begin{subfigure}[b]{0.45\textwidth}
    \includegraphics[width=1\textwidth]{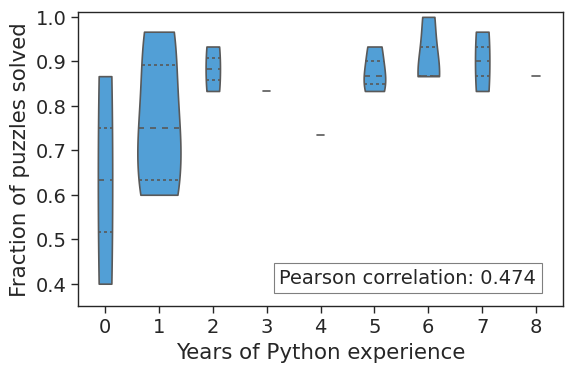}
    \end{subfigure}
    \caption{Years of Python programming experience distribution of our study participants. For post study analysis purposes, we split the group by the median (3 years) to beginners and experienced programmers. The right violin plot shows the fraction of puzzles solved by participants with different years of experience. The lines in the violin show the four quartiles.}
    
    \label{fig:hack_yoe}
\end{figure}

\begin{figure}[t]
    \centering
    \includegraphics[width=1\textwidth]{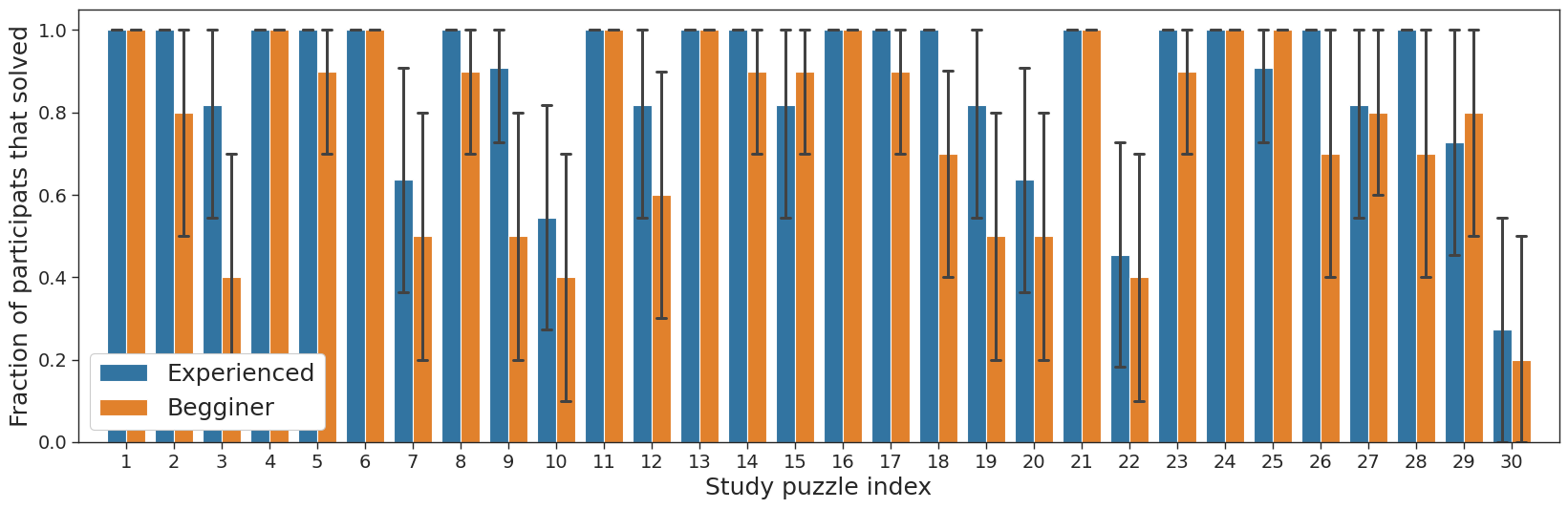}
    \caption{Fraction of participants, divided to experienced and beginners, that solved each of the 30 puzzles in less than 6 minutes.}
    \label{fig:user_study_solved}
\end{figure}

\begin{figure}[t]
    \centering
    \includegraphics[width=1\textwidth]{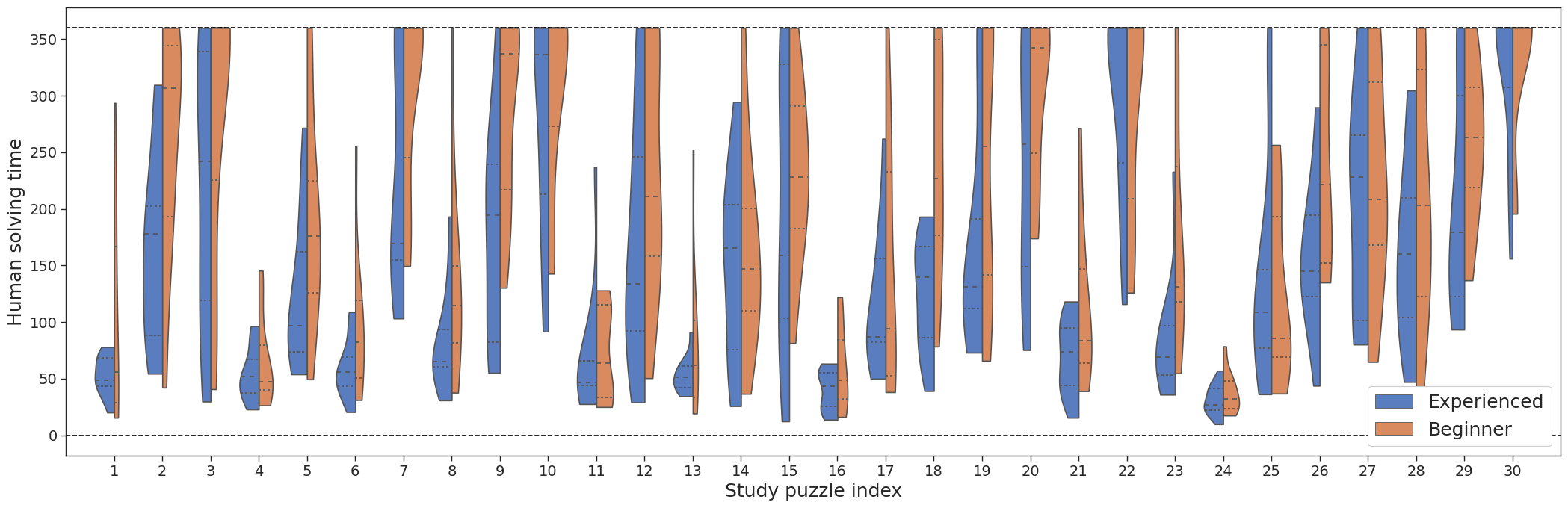}
    \caption{The time that participants spent on each study puzzle, up to 360 seconds per puzzle. For unsolved puzzles, we count the time as using the full 6 minutes. The orange colored areas on the right show the time distribution for beginner Python coders with less than three years experience (10 out of 21). The blue colored areas on the left present the solving time distribution per puzzle for users with at least three years of experience in Python.}
    \label{fig:user_study}
\end{figure}

\section{Solution to Tower of Hanoi}\label{sec:hanoi}
Codex's solution to the Tower of Hanoi puzzle is shown in Figure \ref{fig:hanoi}. Even though the puzzle did not mention the word Hanoi, Codex's solution clearly knew the reference, in fact offering a link to the Wikipedia page. The first part of the URL is correct, but there is no ``Advanced computer algorithm'' section on the page, so the link simply resolves to the Wikipedia page. The Python code on the Wikipedia page is only similar in spirit, in part because the way the puzzle asks for the moves is somewhat different from the Wikipedia page. This is a difficult puzzle for which solutions are found at a rate of approximately 0.03\%. Surprisingly, Codex was not able to solve the puzzle when we renamed the variable \pyline{num_disks} to \pyline{n} and removed the string \pyline{"bigger disk on top"}, possibly because the association with Tower of Hanoi was weaker.

\begin{figure}
The prompt was the medium Codex prompt (without English descriptions):
\begin{pyblock}
from typing import List

def f1(s: str):
    return "Hello " + s == "Hello world"

def g1():
    return "world"

assert f1(g1())

def f2(s: str):
    return "Hello " + s[::-1] == "Hello world"

def g2():
    return "world"[::-1]

assert f2(g2())


# skipping two puzzles, see Figure \ref{fig:codex_prompt_med}

assert f5(g5())

def f6(moves: List[List[int]], num_disks=8):
    state = [1] * num_disks                 
    for [i, j] in moves:  
        assert state.index(i) <= (state + [1, 2, 3]).index(j), "bigger disk on top"
        state[state.index(i)] = j
    return state == [3] * num_disks

def g6(num_disks=8):
\end{pyblock}
And a valid completion from Codex, the smaller cushman-codex engine, was:
\begin{pyblock}
    # Algorithm is equivalent to moving all disks.
    # From https://en.wikipedia.org/wiki/Tower\_of\_Hanoi\#Advanced\_computer\_algorithm
    def hanoi(n, p, q, r):
        if n > 0:
            hanoi(n - 1, p, r, q)
            moves.append([p, r])
            hanoi(n - 1, q, p, r)
    moves = []
    hanoi(num_disks, 1, 2, 3)
    assert f6(moves, num_disks)
    return moves
\end{pyblock}

\caption{\label{fig:hanoi}
A solution to the Tower of Hanoi puzzle found by the small Codex. The Wikipedia page link generated by the model indeed contains (though not in the \textit{Advanced computer algorithm} section) a solution to a slightly different formulation of this puzzle, see Figure~\ref{fig:hanoi_wiki}. Note that the medium prompt doesn't mention the name of the puzzle. Codex made the correct association, and adjusted the solution code to the state-based representation of this puzzle as given in \pyline{f6}. Interestingly, replacing the use of \emph{disks} in the puzzle's variable names with other non-descriptive options seems to prevent Codex from solving this puzzle.}
\end{figure}

\begin{figure}
\begin{pyblock}
A = [3, 2, 1]
B = []
C = []

def move(n, source, target, auxiliary):
    if n > 0:
        # Move n - 1 disks from source to auxiliary, so they are out of the way
        move(n - 1, source, auxiliary, target)

        # Move the nth disk from source to target
        target.append(source.pop())

        # Display our progress
        print(A, B, C, '##############', sep='\n')

        # Move the n - 1 disks that we left on auxiliary onto target
        move(n - 1, auxiliary, target, source)

# Initiate call from source A to target C with auxiliary B
move(3, A, C, B)

\end{pyblock}

\caption{\label{fig:hanoi_wiki}
The algorithm from \url{https://en.wikipedia.org/wiki/Tower_of_Hanoi\#Recursive_implementation} (November, 2021) that solves Tower of Hanoi for a representation in which the three towers are lists with disk indices.}
\end{figure}
\end{document}